\documentclass[10pt,letterpaper]{article}

\usepackage[margin=0.75in,top=1in,bottom=1in]{geometry}
\usepackage[hyphens]{url}
\usepackage{graphicx}
\usepackage{natbib}
\usepackage{caption}

\usepackage{amsmath,amssymb,amsthm,mathtools}
\usepackage{booktabs}
\usepackage{array}
\usepackage{xcolor}
\usepackage{colortbl}
\usepackage{subcaption}
\usepackage{multirow}
\usepackage{algorithm}
\usepackage{algpseudocode}
\usepackage{fancyvrb}
\usepackage{tcolorbox}
\tcbuselibrary{breakable}
\usepackage{tikz}
\usetikzlibrary{arrows.meta,positioning,calc,fit,decorations.pathreplacing}
\usepackage{microtype}

\usepackage{titlesec}

\titleformat{\section}
  {\centering\bfseries\fontsize{11}{13.2}\selectfont}
  {\thesection.}{0.4em}{}
\titlespacing*{\section}{0pt}{10pt plus 2pt minus 2pt}{5pt plus 1pt}

\titleformat{\subsection}
  {\bfseries\normalsize}
  {\thesubsection}{0.4em}{}
\titlespacing*{\subsection}{0pt}{8pt plus 2pt minus 2pt}{3pt plus 1pt}

\titleformat{\subsubsection}
  {\bfseries\itshape\normalsize}
  {\thesubsubsection}{0.4em}{}
\titlespacing*{\subsubsection}{0pt}{6pt plus 2pt}{2pt}

\titleformat{\paragraph}[runin]
  {\bfseries\normalsize}
  {}{0pt}{}[.\enspace]
\titlespacing*{\paragraph}{0pt}{4pt plus 1pt}{0pt}

\usepackage[colorlinks=true,linkcolor=blue!60!black,
            citecolor=blue!60!black,urlcolor=blue!60!black]{hyperref}
\usepackage[capitalize,nameinlink]{cleveref}

\crefname{figure}{Fig.}{Figs.}
\Crefname{figure}{Figure}{Figures}
\crefname{table}{Table}{Tables}
\Crefname{table}{Table}{Tables}
\crefname{equation}{Eq.}{Eqs.}
\Crefname{equation}{Equation}{Equations}
\crefname{section}{Sec.}{Secs.}
\Crefname{section}{Section}{Sections}
\crefname{appendix}{App.}{Apps.}
\Crefname{appendix}{Appendix}{Appendices}
\crefname{algorithm}{Algorithm}{Algorithms}
\Crefname{algorithm}{Algorithm}{Algorithms}
\crefname{proposition}{Proposition}{Propositions}
\Crefname{proposition}{Proposition}{Propositions}
\crefname{corollary}{Corollary}{Corollaries}
\Crefname{corollary}{Corollary}{Corollaries}
\crefname{lemma}{Lemma}{Lemmas}
\Crefname{lemma}{Lemma}{Lemmas}
\crefname{theorem}{Theorem}{Theorems}
\Crefname{theorem}{Theorem}{Theorems}
\crefname{definition}{Definition}{Definitions}
\Crefname{definition}{Definition}{Definitions}
\crefname{remark}{Remark}{Remarks}
\Crefname{remark}{Remark}{Remarks}

\providecommand{\texorpdfstring}[2]{#1}

\makeatletter
\@ifundefinedcolor{linkblue}{\definecolor{linkblue}{RGB}{34,79,135}}{}
\@ifundefinedcolor{codegray}{\definecolor{codegray}{RGB}{244,245,247}}{}
\@ifundefinedcolor{tblhead}{\definecolor{tblhead}{RGB}{240,243,247}}{}
\@ifundefinedcolor{tblgroup}{\definecolor{tblgroup}{RGB}{247,248,250}}{}
\@ifundefinedcolor{tblours}{\definecolor{tblours}{RGB}{236,244,255}}{}
\@ifundefinedcolor{tblref}{\definecolor{tblref}{RGB}{248,248,248}}{}
\makeatother

\providecommand{\R}{\mathbb{R}}
\providecommand{\E}{\mathbb{E}}
\providecommand{\Var}{\mathrm{Var}}

\providecommand{\dpi}{d_\pi}
\providecommand{\bigpo}{\textnormal{\scshape BiGPO}}
\providecommand{\bipace}{\textnormal{\scshape BiPACE}}
\providecommand{\gigpo}{\textnormal{\scshape GiGPO}}
\providecommand{\hgpo}{\textnormal{\scshape HGPO}}

\providecommand{\mico}{\textnormal{\scshape MICo}}

\providecommand{\code}[1]{\colorbox{codegray}{\texttt{\small #1}}}

\providecommand{\tabstretch}{\renewcommand{\arraystretch}{1.15}}
\providecommand{\tblgroup}[2]{\rowcolor{tblgroup}\multicolumn{#1}{@{}l}{\emph{#2}}}

\theoremstyle{plain}
\newtheorem{proposition}{Proposition}

\newtheorem{corollary}{Corollary}

\theoremstyle{definition}

\newtheorem{remark}{Remark}

\title{\textbf{\bipace{}: Bisimulation-Guided Policy Optimization\\
with Action Counterfactual Estimation for LLM Agents}}

\author{%
  Hanyang Wang\textsuperscript{1,\dag}\quad
  Weijieying Ren\textsuperscript{3}\quad
  Yuxiang Zhang\textsuperscript{2}\quad
  Ding Cao\textsuperscript{4}\\[0.2em]
  Zhizhao Zeng\textsuperscript{5}\quad
  Ke Zeng\textsuperscript{5}\quad
  Tianxiang Zhao\textsuperscript{2,*}\\[0.4em]
  \textsuperscript{1}University of Chicago\quad
  \textsuperscript{2}The Hong Kong University of Science and Technology (Guangzhou)\\[0.15em]
  \textsuperscript{3}Stanford University\quad
  \textsuperscript{4}University of Science and Technology of China\quad
  \textsuperscript{5}Meituan\\[0.3em]
  \small\texttt{hanyangw@uchicago.edu}\quad
  \small\texttt{wjyren@stanford.edu}\quad
  \small\texttt{yxzhang25128@gmail.com}\quad
  \small\texttt{caoding@mail.ustc.edu.cn}\\[0.1em]
  \small\texttt{\{zengzhizhao,zengke02\}@meituan.com}\quad
  \small\texttt{tianxiangz@hkust-gz.edu.cn}\\[0.2em]
  \small\textsuperscript{\dag}First author\quad
  \textsuperscript{*}Corresponding author
}
\date{}

\begin{document}

\twocolumn[
  \maketitle
  \vspace{-0.5em}
  \begin{center}
  \begin{minipage}{0.92\textwidth}
  \begin{abstract}
Stepwise group-based RL is an attractive way to train long-horizon LLM
agents without a learned critic: it reuses multiple sampled rollouts to
estimate local advantages. Its weakness is less visible but more
fundamental: every group-relative estimator assumes that the steps it
compares are equivalent for credit assignment. We show that current
agentic variants violate this assumption through a \emph{state-action
credit mismatch}. The observation-hash partition is overly fine on
the state side, creating singleton groups with zero step-level signal,
while a single within-group mean is too coarse on the action side,
mixing state-value estimation with action-specific credit. We introduce
\textbf{\bipace{}} (\emph{Bisimulation-Guided Policy Optimization with
Action Counterfactual Estimation}), a drop-in advantage estimator that
fixes both sides without adding a critic, auxiliary loss, or extra
rollouts. \textbf{\bigpo{}} clusters steps by cosine distance in the
actor's own hidden-state geometry, an empirical, policy-induced proxy for
bisimulation that substantially lowers the singleton rate left by
observation hashing. \textbf{PACE} then recenters returns within each
behavioral cluster using action-conditioned peer baselines; its Q-style
instance estimates a local $\widehat Q(s,a)-\widehat V(s)$
nonparametrically. On \textsc{ALFWorld}/Qwen2.5-7B,
\bipace$_{\text{Q}}$ raises overall validation success from
\gigpo{}'s reported $90.8$ to $\mathbf{97.1{\pm}0.9}$ over three
seeds, and crosses the $95\%$ threshold on every seed, which \gigpo{}
never does within the same budget. On Qwen2.5-1.5B it reaches
$\mathbf{93.5{\pm}1.2}$ versus \gigpo{}'s $86.7$, and on
\textsc{WebShop} and \textsc{TextCraft} it improves over GRPO and
\gigpo{} at both model scales. The change is small in systems terms:
the measured BiPACE-specific share is $11.3\%$ of a single
\textsc{ALFWorld}/Qwen2.5-7B training-step wall time. Yet it changes the
estimator's comparison unit from surface identity to approximate
behavioral equivalence plus action-side counterfactuals.
The code is available at \url{https://github.com/TianxiangZhao/BiPACE}.
  \end{abstract}
  \end{minipage}
  \end{center}
  \vspace{0.5em}
]


\begin{figure*}[t]
\centering
\includegraphics[width=\textwidth]{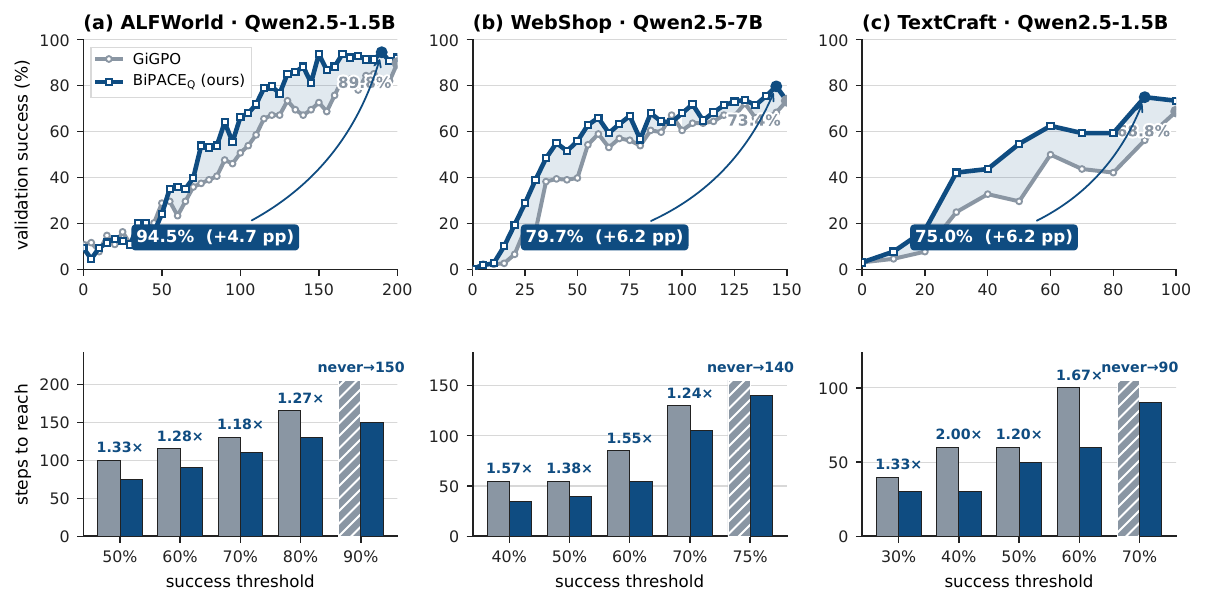}
\captionsetup{font=footnotesize}
\caption{\textbf{\bipace$_{\text{Q}}$ vs.\ \gigpo{} across benchmarks
and model scales.}
\textbf{Top:} validation success over training; dots and badges mark
each method's peak and \bipace$_{\text{Q}}$'s gap over \gigpo{}.
\textbf{Bottom:} steps to reach a fixed success threshold (lower is
better), with speedups
$\text{steps}_{\gigpo{}}/\text{steps}_{\mathrm{BiPACE}_{\text{Q}}}$;
hatched bars mark thresholds never reached. Multi-seed aggregates are
in \cref{tab:main,tab:textcraft}.}
\label{fig:teaser}
\end{figure*}

\section{Introduction}\label{sec:intro}

Reinforcement learning (RL) post-training of large language models has
recently moved beyond single-turn reasoning into the harder
\emph{agentic} regime: long-horizon, partially observed, multi-turn
interaction with tools, web pages, simulated households, and games.
The central obstacle is assigning a sparse terminal reward to the
intermediate decisions that made the trajectory succeed or fail~\citep{wang2025text2grad}.
Group-based RL
methods such as RLOO-style leave-one-out
estimators~\citep{kool2019buy,ahmadian2024back} and
GRPO~\citep{shao2024deepseekmath} are appealing because they avoid a
learned value network. Recent agentic variants such as
\gigpo{}~\citep{feng2026group} and \hgpo{}~\citep{he2026hierarchy} push
this idea to the step level by comparing rollout steps inside groups.
Their performance, however, depends on a choice that is often treated
as an implementation detail: which steps are grouped together for that
comparison.

These estimators share an implicit assumption: if two step records are
placed in the same group, then they are interchangeable for credit
assignment. In long-horizon agent environments, this assumption fails
in two coupled ways. \textbf{State side:} observation identity is a
convenient but overly sparse proxy for value equivalence. Group
baselines are reliable when grouped states share continuation value, a
condition formalized by \emph{bisimulation}
\citep{givan2003equivalence,ferns2004metrics}; observation keys impose
a strictly finer equivalence than bisimulation requires, splitting many
reusable states into isolated singletons. \textbf{Action side:} even
when states are comparable, the usual within-group mean assigns the
same baseline to all actions, ignoring that different actions from the
same state can induce different futures. We call this two-sided
failure the \emph{state-action credit mismatch}.

This mismatch is measurable during training via the singleton
fraction, independently of the final task reward. In our \gigpo{}
reproduction on \textsc{ALFWorld}, $34.2\%$ of step groups are
singletons at iteration 10 and the fraction remains $20.7\%$ at
iteration 140.
Since singleton clusters produce zero step-level advantage, exact
observation hashes discard local signal when the policy needs it most.
On \textsc{TextCraft}, where observations are sparser, exact hashes
isolate even more records and expose fewer matched pairs (detailed
in \cref{sec:mechanism}).

We propose \bipace{} (\emph{Bisimulation-Guided Policy Optimization
with Action Counterfactual Estimation}), a drop-in advantage estimator
that treats step-level credit as two local problems: state aggregation
and action-conditioned credit assignment. \textbf{On the state side},
\bigpo{} replaces observation hashing with cosine clustering on the
actor's normalized hidden state $\phi_\theta(s_t)$ at a fixed late
layer (\cref{app:hparams,app:scale-calibration}), an empirical proxy
for the behavioral metrics of \citet{castro2021mico}. \textbf{On the action side},
PACE partitions each behavioral cluster by the executed action and
augments the cluster-mean baseline with a same-action peer estimate,
forming a nonparametric $\widehat{Q}(s,a)-\widehat{V}(s)$ advantage
inside each cluster. The two halves are coupled: PACE requires behaviorally comparable
state peers, which BiGPO provides.
Our main contributions are summarized as follows:
\begin{itemize}
\item \textbf{\emph{Identifying state-action credit mismatch.}}
  We show that stepwise group-based RL conflates state aggregation with
  action-conditioned credit assignment, and that exact observation
  hashing is the wrong state equivalence relation, splitting reusable
  states into singletons that carry no step-level signal.
\item \textbf{\emph{Proposing a drop-in advantage estimator.}}
  We introduce \bipace{}, which makes two local replacements:
  \bigpo{} clusters actor-hidden fingerprints as a policy-induced
  bisimulation proxy, and PACE adds an action-conditioned peer baseline
  inside each cluster.
\item \textbf{\emph{Analyzing the estimator.}}
  We bound the state-side bias by $O(\varepsilon)$ under a
  \mico{}-Lipschitz assumption, recover \gigpo{} as the
  $\varepsilon{=}0$ limit, quantify the singleton signal loss, and show
  Q-style PACE is exact under exact bisimulation.
\item \textbf{\emph{Achieving strong empirical performance.}}
  \bipace$_{\text{Q}}$ gains $+6.3$pp on \textsc{ALFWorld}/Qwen2.5-7B
  ($97.1{\pm}0.9$ vs.\ $90.8$) and $+6.8$pp on 1.5B, and improves over
  GRPO and \gigpo{} on \textsc{WebShop} and \textsc{TextCraft} at both
  scales, at only $11.3\%$ step overhead.
\end{itemize}

\section{Related Work}\label{sec:related}

{\sloppy
\paragraph{Group-relative RL for LLM agents}
\bipace{} builds on critic-free group-relative RL, including
GRPO~\citep{shao2024deepseekmath}, \gigpo{}~\citep{feng2026group}, and
\hgpo{}~\citep{he2026hierarchy}. These methods compare sampled
returns inside groups but keep the state equivalence relation discrete;
\bipace{} replaces that relation with a policy-induced behavioral
partition. The state side follows the value-preserving bisimulation
view~\citep{ferns2004metrics,castro2021mico,zhang2020learning}, while
PACE gives a nonparametric analogue of action-conditioned
counterfactual baselines studied in COMA/CCPO and related work~\citep{foerster2018counterfactual,li2026counterfactual}. Other
agent-credit methods alter the learning signal or optimizer~\citep{tan2026hindsight,liu2025agentic,wei2025reinforcing,yu2026dapo};
\bipace{} instead changes which step records are compared. Extended
discussion appears in \cref{app:extended-related}.
\par}

\section{Method: \bipace{}}\label{sec:method}

This section first isolates the estimation issue that \bipace{}
targets, then describes the two local replacements that constitute
the method.

\subsection{Estimator setup}
\label{sec:method-setup}
For each prompt group $p$, GRPO samples trajectories
$\{\tau^{(g)}\}_{g=1}^G$ and standardizes terminal returns within the
group,
\begin{equation}
A^{\mathrm{ep}}(\tau^{(g)}) =
\frac{R(\tau^{(g)})-\mu_p}{\sigma_p+\delta},
\quad \mu_p,\sigma_p \text{ over } \{R(\tau^{(g)})\}_{g=1}^G .
\label{eq:episode-adv}
\end{equation}
\gigpo{} adds a step-level term by collecting all step records in the
same prompt group and partitioning them by exact observation hash:
$\mathcal{C}_p = \big\{\,\{i:\mathrm{hash}(s^{(i)})=h\}: h\in \mathrm{Hash}(\{s^{(i)}\}_{i\in p})\,\big\}$.
For each cluster $C\in\mathcal{C}_p$, it normalizes the return-to-go
$R_t^{(i)}$ locally:
\begin{equation}
A^{\mathrm{step}}(i) =
\frac{R_t^{(i)}-\mu_C}{\sigma_C+\delta}, \quad i\in C .
\label{eq:step-adv}
\end{equation}
\bipace{} keeps this training loop intact and replaces only the
partition and the local baseline used by \cref{eq:step-adv}. Extra background on the agent decision process,
GRPO, and bisimulation appears in \cref{app:preliminaries}.

\subsection{The state-action credit mismatch}
\label{sec:method-mismatch}
\gigpo{} improves over trajectory-level GRPO by computing a
step-level advantage within groups of rollout steps that share the
same current observation. This design assumes that two step records in
the same group are exchangeable for credit assignment. In agentic
tasks, the assumption breaks in two complementary ways.

\textbf{State side.}
Exact observation identity is too strict a proxy for value equivalence:
observations that differ only in surface form land in different groups
even when they induce the same continuation value. When the exact-key
partition produces singleton groups, the within-group baseline in
\cref{eq:step-adv} degenerates to zero and the step contributes no
step-level gradient.

\textbf{Action side.}
Even when a group contains behaviorally comparable states, computing a
single cluster mean evaluates every step against the same number,
regardless of which action was taken. Two actions from the same state
neighborhood can lead to different futures, so their advantages should
differ; but subtracting a common mean cannot distinguish between them.
The desired quantity is the local advantage $Q(s,a)-V(s)$: the cluster
mean serves as a state-value estimate $V(s)$, and pooling same-action
peers yields an action-value estimate $Q(s,a)$, isolating
action-specific credit without changing the state-level baseline.

\bipace{} addresses these two sides jointly. On the state side,
\bigpo{} replaces the observation-hash partition with a behavioral
partition derived from the actor's own hidden states. On the action
side, PACE augments the cluster-mean baseline with a same-action peer
estimate, forming the $Q(s,a)-V(s)$ advantage inside each cluster.

\subsection{BiPACE overview}
\label{sec:method-overview}

\begin{figure*}[t]
\centering
\includegraphics[width=\textwidth]{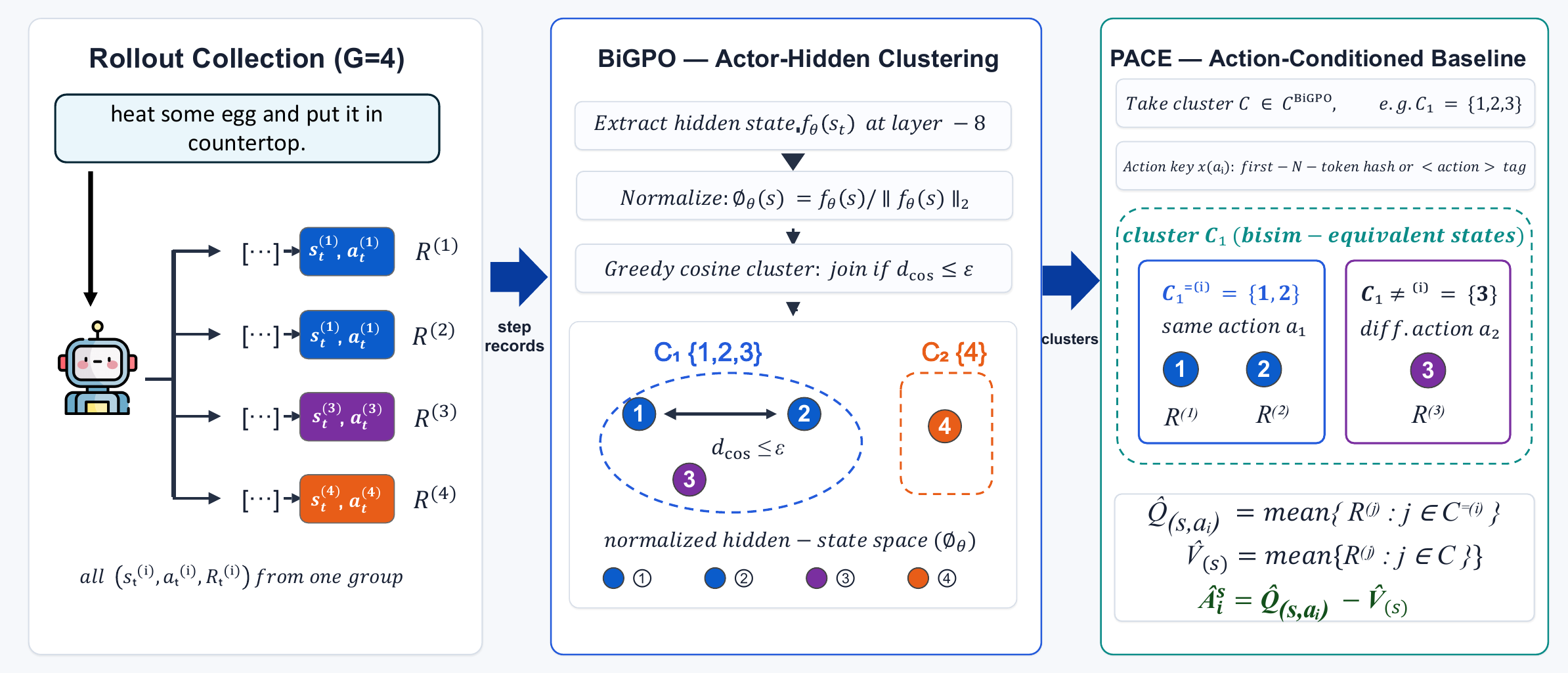}
\captionsetup{font=footnotesize}
\caption{\textbf{Method overview.}
\bipace{} makes two local replacements to the GiGPO step-level
estimator. \textbf{Left:} A prompt group provides step records
($s_t^{(i)}, a_t^{(i)}, R_t^{(i)}$) across $G$ rollouts; chip colors
encode bisimulation class. \textbf{Middle:} \bigpo{} extracts the
actor's normalized hidden state $\phi_\theta(s)$ at a fixed late layer
and clusters by cosine distance, forming behavioral state neighborhoods
$\mathcal{C}_1,\mathcal{C}_2,\ldots$ \textbf{Right:}
PACE splits each cluster by the executed action key and computes a
per-action peer baseline; the Q-style form estimates $\widehat Q(s,a)-\widehat V(s)$.
Only the step-level advantage changes; the PPO objective is unchanged.}
\label{fig:method-overview}
\end{figure*}

\begin{tcolorbox}[
  colback=tblours!55,
  colframe=linkblue!38,
  boxrule=0.55pt,
  arc=1mm,
  left=7pt,
  right=7pt,
  top=6pt,
  bottom=6pt]
\textbf{\bipace{} in one line.}
For each prompt group, \bipace{} first clusters step records by the
actor's own hidden-state geometry (\bigpo{}), then computes the
step-level baseline inside each cluster conditional on the executed
action (PACE). Everything outside the advantage estimator is identical
to \gigpo{}.
\end{tcolorbox}

\Cref{fig:method-overview} traces the full pipeline; the two components
compose rather than simply stack, as PACE's per-action baseline requires
a behaviorally coherent peer pool that \bigpo{} provides.
\Cref{sec:method-bigpo,sec:method-pace} detail each in turn.

\subsection{State-side grouping with \bigpo{}}
\label{sec:method-bigpo}
The key insight behind \bigpo{} is that the actor's own hidden states
already encode behavioral similarity: observations the policy processes
identically cluster tightly in late-layer representation space,
regardless of their surface form.
Exact observation hashing ignores this geometry (treating any
surface-distinct observations as incomparable even when the policy
responds to them identically; \cref{sec:method-mismatch}), and
routinely discards the local signal that step-level RL is designed to
exploit.
\bigpo{} replaces the exact hash with a soft behavioral partition
derived from this representation.

Concretely, let
$f_\theta : \mathcal{S}\to\R^D$ be the function that maps the current
prompt $s_t$ to the actor LLM's hidden state at the final prompt
token, taken at a fixed late intermediate layer chosen once per
backbone (layer $-8$ for Qwen2.5-7B, $-12$ for Qwen2.5-1.5B;
calibration details in \cref{app:hparams,app:scale-calibration}). We use the normalized representation
$\phi_\theta(s) = f_\theta(s)/\|f_\theta(s)\|_2$.
This representation moves with the policy being optimized and requires
no learned critic or auxiliary model; it is obtained via a dedicated
actor forward pass with hidden-state extraction enabled, whose cost is
measured in \cref{sec:computational-budget}.

For each prompt group $p$, \bigpo{} computes the partition
\begin{equation}
\mathcal{C}_p^{\bigpo} \;=\; \mathrm{Cluster}_\varepsilon\big(
\{ \phi_\theta(s^{(i)}) : i \in p \}, \; d_{\cos}\big),
\label{eq:bigpo-clusters}
\end{equation}
where $d_{\cos}(u,v)=1-u^\top v$ and
$\mathrm{Cluster}_\varepsilon$ is a single-pass greedy procedure: each
record joins the nearest existing centroid if its cosine distance is
$\le\varepsilon$, otherwise it seeds a new cluster, and the joined
centroid is updated online; full pseudocode is in \cref{app:greedy}.
The \gigpo{} step advantage in \cref{eq:step-adv} then applies
unchanged, with $\mathcal{C}_p$ replaced by $\mathcal{C}_p^{\bigpo}$. \Cref{sec:theory} bounds the resulting bias: under a
\mico{}-Lipschitz assumption on the embedding, the replacement trades an
$O(\varepsilon)$ bias for many more non-singleton step groups.

\textbf{Relation to prior methods.} Setting $\varepsilon{=}0$ and using a one-hot observation hash as $\phi$ exactly recovers \gigpo{}; replacing the hash with a history-aware signature recovers \hgpo{}. The \gigpo{}/\hgpo{} family thus uses hand-designed discrete fingerprints, while \bigpo{} uses the policy's own continuous fingerprint with coarsening controlled by $\varepsilon$.

\textbf{Embedder design space.} The estimator only requires a fingerprint $\phi$, leaving the embedder open. We examine
two backends along an effort/fidelity axis: \emph{HashNgram}, a
zero-dependency character-$n$-gram hash that groups by lexical
surface form, and \emph{Actor-Hidden}, the policy LLM's own hidden
state. HashNgram is a policy-agnostic control: it isolates whether
the gain comes from the policy-induced geometry or merely from
coarsening the partition; Actor-Hidden
is the main method because it changes with the policy being optimized.

\subsection{Action-conditioned baseline with PACE}
\label{sec:method-pace}
PACE realizes the $Q(s,a){-}V(s)$ decomposition identified in
\cref{sec:method-mismatch}: within each behavioral cluster, the
cluster mean estimates $V(s)$ and a same-action peer mean estimates
$Q(s,a)$, isolating action-specific credit without changing the
state-level baseline.

Concretely, PACE splits each behavioral cluster by the executed
action. For each action $a$, let $\kappa(a)\in\mathbb{Z}$ be an action key. We use two
practical keys: \emph{first-token}, the hash of the first $N{=}8$
response tokens, and \emph{action-tag}, the hash of the body of
\texttt{<action>...</action>}. The latter is semantically cleaner when
the environment exposes parseable actions; the former is cheap and
robust when parsing is unavailable.

Inside a cluster $C\in\mathcal{C}_p^{\bigpo}$, write
$k_i=\kappa(a^{(i)})$, $C^{=}(i)=\{j\in C:k_j=k_i\}$, and
$C^{\neq}(i)=C\setminus C^{=}(i)$. PACE instantiates two
nonparametric estimators:
\begingroup
\small
\setlength{\abovedisplayskip}{3pt}
\setlength{\belowdisplayskip}{3pt}
\setlength{\jot}{1pt}
\begin{align}
\hat A^{\text{diff}}(i)
&= R^{(i)} -
\tfrac{1}{|C^{\neq}(i)|}\!\!\sum_{j\in C^{\neq}(i)}\!\! R^{(j)}
\label{eq:pace-diff}\\
\hat A^{\text{q}}(i)
&= \widehat{Q}(s,a_i)-\widehat{V}(s)
\label{eq:pace-qstyle}\\
\widehat{Q}(s,a_i)
&= \tfrac{1}{|C^{=}(i)|}\sum_{j\in C^{=}(i)} R^{(j)}\\
\widehat{V}(s)
&= \tfrac{1}{|C|}\sum_{j\in C} R^{(j)}.
\end{align}
\endgroup
The diff-peer form compares each action against peers that took a
different action from the same state neighborhood; the Q-style form
directly estimates $\widehat Q(s,a)-\widehat V(s)$.

Fallbacks keep the estimator well-defined. Singleton clusters retain
$\hat A^{\text{step}}=0$ as in \gigpo{}. The diff-peer form falls back
to RLOO leave-one-out when $|C^{\neq}(i)|=0$; the Q-style form falls
back when $|C^{=}(i)|=1$. Empirically, the Q-style branch is
non-degenerate on \textsc{ALFWorld} and gives the best completed
variant in \cref{sec:pace-ablation}; in environments with larger
effective action spaces, the diff-peer form is the safer default.

Combining with the episode term gives the final per-token advantage
\begin{equation}
A^{\bipace}(i) \;=\; A^{\mathrm{ep}}(i) \;+\; w\cdot \hat A^{S}(i),
\label{eq:bipace-final}
\end{equation}
where $\hat A^{S}$ is the selected step-level estimator and $w$ is the
same fixed weight used by \gigpo{}. The PPO surrogate is unchanged.
Implementation details are in \cref{app:impl} and the
computational-budget discussion below.

\section{Experiments}\label{sec:experiments}

We organize the experiments around four questions:
\textbf{RQ1:} Does \bipace{} improve end-task success under the same
rollout budgets and model scales?
\textbf{RQ2:} Does the improvement come with better sample
efficiency?
\textbf{RQ3:} Do the policy-state groups actually increase usable
step-level interactions?
\textbf{RQ4:} Is the action-conditioned PACE estimator necessary on
top of the state-side partition?
We present RQ3 (mechanism) before RQ1 (end-task results) to
establish the underlying diagnostic before interpreting the headline
numbers; RQ2 and RQ4 follow.
The appendix provides extended related work, background, proofs,
reproducibility details, hyperparameters, prompts, calibration scans,
diagnostic tables, failure modes, and full per-seed results.

\subsection{Setup}
\textbf{Environments.}
\textsc{ALFWorld}~\citep{shridhar2020alfworld},
\textsc{WebShop}~\citep{yao2022webshop}, and
\textsc{TextCraft}~\citep{prasad2024adapt}.
\textbf{Models.} Qwen2.5-\{1.5, 7\}B-Instruct~\citep{yang2024qwen25}.
\textbf{Baselines.} GRPO~\citep{shao2024deepseekmath},
PPO with critic~\citep{schulman2017proximal},
\gigpo{}~\citep{feng2026group}, and \hgpo{}~\citep{he2026hierarchy};
prompting rows from \citet{feng2026group} anchor the benchmark scale.
\textbf{Hardware.} 4$\times$H100 (7B); 2--4$\times$H100 (1.5B).
Full settings, seeds, hyperparameters, and prompt templates are in
\cref{app:hparams,app:prompts}.

\subsection{Mechanism: the singleton tax (RQ3)}
\label{sec:mechanism}
\begin{table}[tb]
\centering
\caption{Singleton fraction on \textsc{ALFWorld}/7B.}
\label{tab:singleton}
\footnotesize
\tabstretch
\setlength{\tabcolsep}{6pt}
\begin{tabular}{@{}l ccc@{}}
\toprule
\rowcolor{tblhead}
Method & iter $10$ & iter $75$ & iter $140$ \\
\midrule
\gigpo{} (obs. hash) & $34.2\%$ & $33.1\%$ & $20.7\%$ \\
\rowcolor{tblours}
\bigpo{} (Actor-Hidden) &
$\mathbf{17.3\%}$ & $\mathbf{17.2\%}$ & $\mathbf{14.1\%}$ \\
\bottomrule
\end{tabular}
\end{table}

\gigpo{}'s observation-hash partition leaves many step records in
singleton clusters. These records receive zero step-level advantage by
construction, so the partition directly controls how much of the
step-level gradient can be used. \Cref{tab:singleton} shows the
mechanism-level change: the policy-induced bisimulation partition
starts with a lower singleton fraction than \gigpo{} has even late in
training. Both rows are measured from our training-log diagnostics.
\bigpo{} entries are $5$-step window means centered at the listed
iteration, averaged over the \bipace$_{\text{Q}}$ seeds whose logs cover
that window; the partition depends only on the state-side clustering,
not on the PACE estimator.

\subsection{End-task performance (RQ1)}
\label{sec:main-results}
\Cref{tab:main} is the main result table. The 7B comparison is the
cleanest completed setting: over three seeds, \bipace$_{\text{Q}}$
raises aggregate val\,$@$max (binary val/success-rate, count-weighted by
the validation set) from \gigpo{}'s $90.8$ to $97.1{\pm}0.9$, while
saturating five of six \textsc{ALFWorld} task families at $100\%$
per-subtask val\,$@$max across all seeds.
All three \bipace$_{\text{Q}}$ seeds individually reach the $95\%$ threshold
(at steps $115$--$135$; \cref{app:full-results}); no \gigpo{} seed does so
within the same $150$-step budget. The per-subtask cells are
diagnostic slices: because each slice takes its own best checkpoint,
the aggregate \emph{All} column is the headline metric. The 1.5B
result is a three-seed transfer check; the smaller backbone converges
more slowly, so its row is reported at a later checkpoint within our
training budget and marked with $\star$ (extended budget). Across
three seeds, overall val\,$@$max is $93.5{\pm}1.2$, with three of six
task families saturated at $100\%$ on every seed; per-seed values are
reported in \cref{app:full-results}.

\begin{table*}[t]
\centering
\caption{Task success rate (\%) on \textsc{ALFWorld}
(\textit{valid-seen}, 6 task families) and \textsc{WebShop}
(\emph{Score}\,/\,\emph{Success}). \emph{All} is the count-weighted
overall val\,$@$max on the binary validation success metric. Reference
rows are reproduced from~\citet{feng2026group} and
\citet{he2026hierarchy}. $\star$\,extended training budget (200 steps
vs.\ 150 for 7B); $\ddagger$\,partial run (fewer completed seeds).}
\label{tab:main}
\scriptsize
\tabstretch
\setlength{\tabcolsep}{2.5pt}
\resizebox{\textwidth}{!}{%
\begin{tabular}{@{}ll ccccccc | cc@{}}
\toprule
& & \multicolumn{7}{c|}{\textsc{ALFWorld}} & \multicolumn{2}{c}{\textsc{WebShop}}\\
\cmidrule(lr){3-9}\cmidrule(lr){10-11}
\rowcolor{tblhead}
Type & Method & Pick & Look & Clean & Heat & Cool & Pick2 & All & Score & Succ. \\
\midrule
\tblgroup{11}{Closed-source prompting (no fine-tuning).} \\
Prompting   & GPT-4o          & $75.3$ & $60.8$ & $31.2$ & $56.7$ & $21.6$ & $49.8$ & $48.0$ & $31.8$ & $23.7$ \\
Prompting   & Gemini-2.5-Pro  & $92.8$ & $63.3$ & $62.1$ & $69.0$ & $26.6$ & $58.7$ & $60.3$ & $42.5$ & $35.9$ \\
\midrule
\tblgroup{11}{Qwen2.5-1.5B-Instruct.} \\
Prompting   & Qwen2.5           & $5.9$  & $5.5$  & $3.3$  & $9.7$  & $4.2$  & $0.0$  & $4.1$  & $23.1$ & $5.2$ \\
Prompting   & ReAct             & $17.4$ & $20.5$ & $15.7$ & $6.2$  & $7.7$  & $2.0$  & $12.8$ & $40.1$ & $11.3$ \\
Prompting   & Reflexion         & $35.3$ & $22.2$ & $21.7$ & $13.6$ & $19.4$ & $3.7$  & $21.8$ & $55.8$ & $21.9$ \\
RL training & PPO (w/ critic)   & $64.8{\pm}3.5$ & $40.5{\pm}6.9$ & $57.1{\pm}4.9$ & $60.6{\pm}6.6$ & $46.4{\pm}4.0$ & $47.4{\pm}1.9$ & $54.4{\pm}3.1$ & $73.8{\pm}3.0$ & $51.5{\pm}2.9$ \\
RL training & GRPO              & $85.3{\pm}1.5$ & $53.7{\pm}8.0$ & $84.5{\pm}6.8$ & $78.2{\pm}7.9$ & $59.7{\pm}5.0$ & $53.5{\pm}5.6$ & $75.2{\pm}3.8$ & $75.8{\pm}3.5$ & $56.8{\pm}3.8$ \\
RL training & \gigpo{}    & $94.4{\pm}5.9$ & $67.5{\pm}4.6$ & $94.8{\pm}3.8$ & $94.4{\pm}7.8$ & $79.8{\pm}4.7$ & $76.4{\pm}5.4$ & $86.7{\pm}1.7$ & $83.5{\pm}1.8$\textsuperscript{$\ddagger$} & $67.4{\pm}4.5$\textsuperscript{$\ddagger$} \\
\rowcolor{tblref}
RL training & \hgpo{}   & \multicolumn{6}{c}{not reported} & $92.8{\pm}1.1$ & $85.6{\pm}2.9$ & $71.5{\pm}4.0$ \\
\rowcolor{tblours}
RL training & \textbf{\bipace$_{\text{Q}}$ (ours)}\textsuperscript{$\star$} & $\mathbf{100.0}$ & $\mathbf{97.4}{\pm}3.8$ & $\mathbf{100.0}$ & $\mathbf{100.0}$ & $\mathbf{96.5}{\pm}3.6$ & $\mathbf{92.0}{\pm}7.9$ & $\mathbf{93.5}{\pm}1.2$ & $\mathbf{85.8}{\pm}1.1$ & $\mathbf{71.9}{\pm}2.1$ \\
\midrule
\tblgroup{11}{Qwen2.5-7B-Instruct.} \\
Prompting   & Qwen2.5           & $33.4$ & $21.6$ & $19.3$ & $6.9$  & $2.8$  & $3.2$  & $14.8$ & $26.4$ & $7.8$ \\
Prompting   & ReAct             & $48.5$ & $35.4$ & $34.3$ & $13.2$ & $18.2$ & $17.6$ & $31.2$ & $46.2$ & $19.5$ \\
Prompting   & Reflexion         & $62.0$ & $41.6$ & $44.9$ & $30.9$ & $36.3$ & $23.8$ & $42.7$ & $58.1$ & $28.8$ \\
RL training & PPO (w/ critic)   & $92.3{\pm}4.0$ & $64.0{\pm}8.4$ & $92.5{\pm}2.4$ & $89.5{\pm}7.0$ & $80.3{\pm}2.0$ & $68.8{\pm}8.3$ & $80.4{\pm}2.7$ & $81.4{\pm}3.1$ & $68.7{\pm}5.1$ \\
RL training & GRPO              & $90.8{\pm}5.1$ & $66.1{\pm}6.7$ & $89.3{\pm}5.4$ & $74.7{\pm}6.9$ & $72.5{\pm}5.4$ & $64.7{\pm}7.3$ & $77.6{\pm}5.2$ & $79.3{\pm}2.8$ & $66.1{\pm}3.7$ \\
RL training & \gigpo{}    & $97.7{\pm}1.6$ & $82.7{\pm}7.9$ & $98.8{\pm}1.6$ & $83.7{\pm}7.2$ & $89.3{\pm}8.2$ & $79.2{\pm}6.6$ & $90.8{\pm}1.3$ & $86.2{\pm}2.6$ & $75.2{\pm}3.8$ \\
\rowcolor{tblref}
RL training & \hgpo{}   & \multicolumn{6}{c}{not reported} & $95.4{\pm}0.6$ & $89.0{\pm}1.0$ & $78.5{\pm}1.4$ \\
\rowcolor{tblours}
RL training & \textbf{\bipace$_{\text{Q}}$ (ours)} & $\mathbf{100.0}$ & $\mathbf{100.0}$ & $\mathbf{100.0}$ & $\mathbf{100.0}$ & $\mathbf{95.3}{\pm}3.7$ & $\mathbf{100.0}$ & $\mathbf{97.1}{\pm}0.9$ & $\mathbf{89.6}{\pm}1.3$ & $\mathbf{79.7}{\pm}3.3$ \\
\bottomrule
\end{tabular}%
}
\end{table*}
\paragraph{\textsc{TextCraft} transfer}\label{sec:textcraft}
\textsc{TextCraft} provides an out-of-domain transfer check with
depth-stratified crafting goals (depth-2: short chains;
depth-3: longer subplans requiring intermediate reuse;
depth-4 omitted due to insufficient validation examples).
We apply the same Q-style PACE recipe; the group sizes, training
windows, and action-key choices are listed in \cref{app:hparams}.

\begin{table}[tb]
\centering
\caption{\textsc{TextCraft} validation success rate (\%): peak val
success within the stated window.}
\label{tab:textcraft}
\scriptsize
\tabstretch
\setlength{\tabcolsep}{4pt}
\begin{tabular}{@{}l l c c c@{}}
\toprule
\rowcolor{tblhead}
Type & Method & overall & depth-$2$ & depth-$3$ \\
\midrule
\tblgroup{5}{Qwen2.5-1.5B-Instruct.} \\
Prompting & Qwen2.5 &
$4.2{\pm}1.5$ & $5.7{\pm}2.0$ & $0.0{\pm}0.0$ \\
RL training & GRPO &
$43.8{\pm}2.2$ & $59.5{\pm}1.9$ & $20.6{\pm}4.2$ \\
RL training & \gigpo{} &
$58.3{\pm}7.8$ & $76.7{\pm}5.4$ & $21.8{\pm}9.6$ \\
RL training & \hgpo{} &
$59.4{\pm}8.0$ & $77.8{\pm}8.6$ & $17.6{\pm}4.6$ \\
\rowcolor{tblours}
RL training & \textbf{\bipace$_{\text{Q}}$ (ours)} &
$\mathbf{65.1{\pm}7.0}$ & $\mathbf{84.9{\pm}3.0}$ &
$\mathbf{29.6{\pm}12.0}$ \\
\midrule
\tblgroup{5}{Qwen2.5-7B-Instruct.} \\
Prompting & Qwen2.5 &
$6.8{\pm}2.7$ & $7.9{\pm}2.7$ & $4.0{\pm}2.9$ \\
RL training & GRPO &
$76.5{\pm}3.1$ & $90.3{\pm}0.8$ & $55.5{\pm}10.0$ \\
RL training & \gigpo{} &
$87.0{\pm}2.4$ & $94.0{\pm}1.6$ & $75.0{\pm}3.5$ \\
RL training & \hgpo{} &
$87.5{\pm}2.3$ & $93.6{\pm}1.8$ & $85.0{\pm}6.5$ \\
\rowcolor{tblours}
RL training & \textbf{\bipace$_{\text{Q}}$ (ours)} &
$\mathbf{91.1{\pm}2.4}$ & $\mathbf{95.7{\pm}2.1}$ &
$\mathbf{87.4{\pm}13.4}$ \\
\bottomrule
\end{tabular}
\end{table}

\Cref{tab:textcraft} gives a small out-of-domain check. Prompting alone
does not solve the transfer setting (${\le}7\%$ overall), and
\bipace$_{\text{Q}}$ is the strongest trained row at both scales. Its
largest margins are on depth-$3$ goals ($+7.8$pp over \gigpo{} on
1.5B, $+12.4$pp on 7B), where intermediate states can lead to several
action-conditioned futures. The \hgpo{} rows land in the \gigpo{} band on overall success at both scales;
\bipace$_{\text{Q}}$ outperforms \hgpo{} by $+3.6$pp at 7B ($91.1$ vs.\ $87.5$) and $+5.7$pp at 1.5B ($65.1$ vs.\ $59.4$).
\paragraph{Policy-state interaction diagnostics}
\begin{figure}[tb]
\centering
\includegraphics[width=\linewidth]{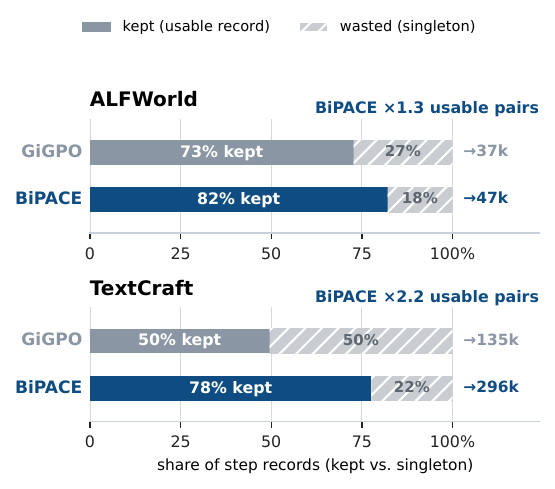}
\captionsetup{font=footnotesize}
\caption{Step-record utilization on \textsc{ALFWorld}/7B and
\textsc{TextCraft}/7B (\emph{kept}: multi-member cluster;
\emph{wasted}: singleton). \bipace{} yields $\times1.3$ usable pairs
on \textsc{ALFWorld} and $\times2.2$ on \textsc{TextCraft};
means over diagnostic seeds, first $130$/$50$ steps.}
\label{fig:bipace-interaction-trends}
\end{figure}
\Cref{fig:bipace-interaction-trends} measures step-record utilization
from paired training-log diagnostics
\citep[funnel diagnostic style of][]{he2026hierarchy}. On
\textsc{ALFWorld}/7B, \bipace{} wastes fewer records to singletons and
clears roughly $\times1.3$ as many matched pairs. On
\textsc{TextCraft}/7B, where exact hashes are sparser, it keeps the
singleton share near $20$--$25\%$ and exposes roughly twice as many
matched pairs. The learned actor representation therefore creates
larger reusable state pools for PACE's action-conditioned baseline.
\subsection{Sample efficiency (RQ2)}
\label{sec:sample-efficiency}
{\sloppy
Sample efficiency is summarized in \cref{fig:teaser}: the bottom row
reports steps-to-threshold (lower is better) for one seed per method
on all three benchmarks, and \bipace$_{\text{Q}}$ reaches every
threshold first. On \textsc{ALFWorld}/1.5B (the smaller backbone is
sub-saturated at the 7B budget, so both methods are run for an
extended budget; \cref{fig:teaser}a), \bipace$_{\text{Q}}$ is
$1.18$--$1.33\times$ faster across the $50$--$80\%$ band and is the
only method to cross $90\%$ within the budget. The same pattern holds on
\textsc{ALFWorld}/7B: \bipace$_{\text{Q}}$ is $1.05$--$1.25\times$
faster across the $60$--$95\%$ band and is the only method to cross
$95\%$ within the $150$-step budget, at step $100$ on the best seed;
all three seeds cross within the budget. On
\textsc{WebShop}/7B and \textsc{TextCraft}/1.5B
(\cref{fig:teaser}b,c), the speedups reach $1.57\times$ and
$2.00\times$, and the top threshold in each panel is reached only by
\bipace$_{\text{Q}}$ within the budget.
\par}
\subsection{Action-side ablation: PACE (RQ4)}
\label{sec:pace-ablation}
\begin{figure}[t]
\centering
\includegraphics[width=0.75\linewidth]{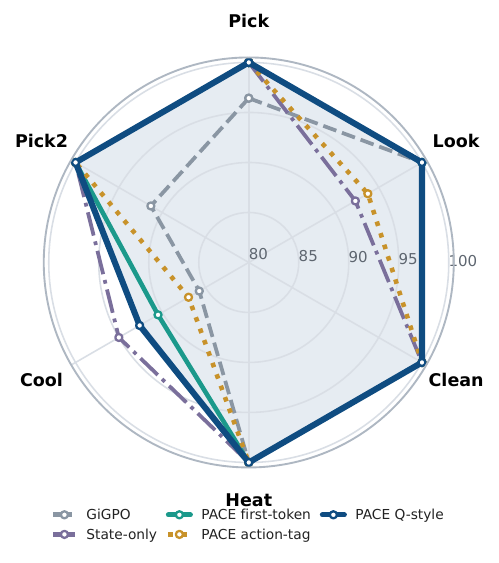}
\captionsetup{font=footnotesize}
\caption{Matched-seed per-task PACE ablation on \textsc{ALFWorld}/7B
(validation peak, \%); radial axis truncated to $80$--$100$.}
\label{fig:pace-radar}
\end{figure}
The state-only variant isolates \bigpo{}'s partition; PACE tests
whether action-conditioning within each cluster adds further gain.
We compare three variants atop Actor-Hidden clustering on
\textsc{ALFWorld}/7B (\cref{fig:pace-radar}; full numeric table in \cref{app:pace-diag}):
\textbf{first-token} (diff-peer mean, \cref{eq:pace-diff}, keyed on
the first $N{=}8$ tokens),
\textbf{action-tag} (same estimator, \texttt{<action>}-body key), and
\textbf{Q-style} (action-tag key with $\widehat{Q}(s,a){-}\widehat{V}(s)$,
\cref{eq:pace-qstyle}).

Any action-conditioning on top of state-only clustering helps
(\cref{fig:pace-radar}; Look, Cool, and Pick2 are the informative
margins, with most other families near ceiling):
first-token PACE reaches $95.8{\pm}0.4\%$ across three seeds,
confirming that \bigpo{}'s behavioral peers support action-specific
credit.
Swapping to the \texttt{<action>}-body key with the same diff-peer
estimator underperforms first-token by $2.8$pp ($93.0{\pm}1.1\%$;
\cref{app:pace-diag}), suggesting that the first ${\sim}8$ tokens
already disambiguate the command on \textsc{ALFWorld} and that
parsing the full action body adds fragility without benefit.
Q-style is the strongest variant ($97.1{\pm}0.9\%$, three seeds):
keeping $\widehat{Q}$ and $\widehat{V}$ in non-overlapping pools
recovers a cleaner estimate of $A(s,a){=}Q(s,a){-}V(s)$ than
diff-peer's mixed-action baseline.
Replacing the Actor-Hidden fingerprint with a policy-agnostic
character-$n$-gram hash, holding the Q-style PACE estimator
fixed, yields $95.4\%$, which is $2.4$pp above the state-only level but
$1.7$pp below Actor-Hidden (\cref{tab:local-checks}), confirming that
the policy-induced geometry contributes gain beyond what lexical
coarsening alone provides.

\begin{table}[tb]
\centering
\caption{Local checks on \textsc{ALFWorld}/7B with Q-style PACE fixed
(embedder: three seeds; radius: seed $0$).}
\label{tab:local-checks}
\scriptsize
\tabstretch
\setlength{\tabcolsep}{4pt}
\begin{tabular}{@{}l c c c c@{}}
\toprule
\rowcolor{tblhead}
Check & setting A & val\,$@$max & setting B & val\,$@$max \\
\midrule
\textsc{Embedder} & HashNgram & $95.4{\pm}1.2$ &
\textbf{Actor-Hidden} & $\mathbf{97.1{\pm}0.9}$ \\
\midrule
\multirow{2}{*}{\textsc{Radius}} &
$\varepsilon{=}0.05$ & $94.5$ &
$\varepsilon{=}0.10$ & $\mathbf{97.7}$ \\
& $\varepsilon{=}0.15$ & $96.1$ &
$\varepsilon{=}0.20$ & $94.5$ \\
\bottomrule
\end{tabular}
\end{table}

The same recipe transfers to Qwen2.5-1.5B ($93.5{\pm}1.2$
vs.~\gigpo{}'s $86.7{\pm}1.7$, $+6.8$pp).

The default clustering radius is also not brittle. With Q-style PACE
fixed on \textsc{ALFWorld}/7B seed $0$, validation success peaks at the
$\varepsilon{=}0.10$ default but remains within a ${\sim}3$pp band over
$\{0.05,0.10,0.15,0.20\}$ (\cref{tab:local-checks}). A row-mix diagnostic
at the default further confirms that the Q-style estimator is not mostly
falling back: $80.2\%$ of rows enter the PACE branch and multi-member
clusters average $2.76$ distinct action keys (\cref{app:pace-diag}).
The action side carries most of the gain.
Once \bigpo{} supplies behaviorally comparable peers, the multi-seed
PACE variants improve over the \gigpo{} reimplementation by
${+}5.0$pp (first-token) and ${+}6.3$pp (Q-style) on the aggregate
val/success-rate metric; state-only clustering alone contributes
${+}2.2$pp. The same pattern transfers to 1.5B:
\bipace$_{\text{Q}}$ gains ${+}6.8$pp over \gigpo{} on overall
val\,$@$max ($3$ seeds; \cref{tab:main}), consistent with the
7B ${+}6.3$pp gap.

\subsection{Computational budget}
\label{sec:computational-budget}
{\sloppy
The only additional measured work over the base \gigpo{}/GRPO loop is
local to the step-level estimator: an actor-hidden forward pass to
obtain policy-state fingerprints, followed by lightweight PACE
grouping and advantage estimation.

\Cref{fig:actor-hidden-overhead} shows that the per-iteration budget is
still dominated by the shared training loop. Rollout generation alone
takes $197.25$s, and the actor update takes $88.68$s, compared with a
$361.27$s pure training step. The BiPACE-specific components total
$40.70$s, or $11.3\%$ of a step. Almost all of this measured addition is
the optional actor-hidden extraction forward ($40.21$s); the PACE
grouping and action-conditioned advantage estimation costs only $0.49$s
per iteration, $0.14\%$ of the step budget. Thus the estimator changes
which step records are compared for credit assignment while leaving the
dominant rollout, probability, and policy-update costs intact.
To be precise: this forward pass adds computation but no extra
environment interactions (rollouts); the two should not be conflated.
The number of records being grouped scales with the rollout budget as
$O(GT)$ for group size $G$ and horizon $T$, so larger budgets increase
the absolute grouping cost even though the measured constant here is
small.
\par}

\begin{figure}[tb]
\centering
\includegraphics[width=\linewidth]{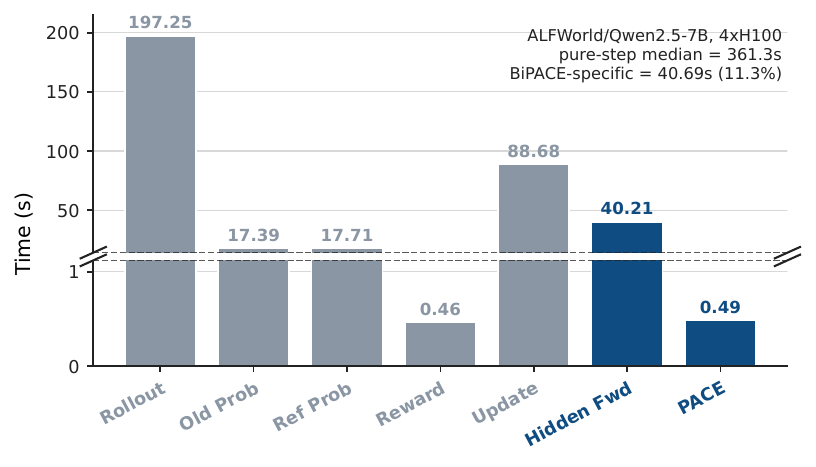}
\captionsetup{font=footnotesize}
\caption{Per-iteration budget on \textsc{ALFWorld}/Qwen2.5-7B
(4$\times$H100); blue bars are BiPACE-specific.}
\label{fig:actor-hidden-overhead}
\end{figure}

\section{Conclusion and Limitations}\label{sec:conclusion}

We have presented \bipace{}, a drop-in advantage estimator for
agentic GRPO that addresses the state-action credit mismatch in
step-level group-based reinforcement learning.
Specifically, \bipace{} introduces two local replacements to the
\gigpo{} estimator: \bigpo{}, which clusters policy-hidden-state
fingerprints as a bisimulation proxy to substantially reduce singleton
groups, and PACE, which adds action-conditioned peer baselines within
each behavioral cluster to recover a local $Q{-}V$ advantage, without
auxiliary models or extra rollouts.
Across three environments and two model scales, \bipace$_{\text{Q}}$
consistently outperforms GRPO and \gigpo{}, and improves over \hgpo{}
on \textsc{ALFWorld}/7B ($+1.7$pp) and \textsc{TextCraft};
all three \textsc{ALFWorld}/7B seeds cross the $95\%$ threshold within
the training budget.

Several directions remain open.
\bipace{} is currently evaluated on text-only environments with
discrete action spaces, and extending to vision-based or
continuous-action settings is a natural next step.
The cosine radius $\varepsilon$ is also fixed via a one-time
calibration scan; allowing it to adapt online as the policy evolves
would sharpen the partition over the course of training.
Richer action representations for PACE beyond the action-tag key
(such as full-text embeddings) could improve counterfactual contrast
in environments with larger action spaces.
Finally, extending the bisimulation-guided grouping to agents that
compress history into a memory module (where direct observation
hashing is intractable) is an interesting avenue for future work.

\clearpage
\bibliographystyle{plainnat}
\bibliography{references}

@article{shao2024deepseekmath,
  title={{DeepSeekMath}: Pushing the limits of mathematical reasoning in open language models},
  author={Shao, Zhihong and Wang, Peiyi and Zhu, Qihao and Xu, Runxin and Song, Junxiao and Bi, Xiao and Zhang, Haowei and Zhang, Mingchuan and Li, YK and Wu, Yang and others},
  journal={arXiv preprint arXiv:2402.03300},
  year={2024}
}

@article{feng2026group,
  title={Group-in-group policy optimization for {LLM} agent training},
  author={Feng, Lang and Xue, Zhenghai and Liu, Tingcong and An, Bo},
  journal={Advances in Neural Information Processing Systems},
  volume={38},
  pages={46375--46408},
  year={2025}
}

@article{he2026hierarchy,
  title={Hierarchy-of-groups policy optimization for long-horizon agentic tasks},
  author={He, Shuo and Feng, Lang and Wei, Qi and Cheng, Xin and Feng, Lei and An, Bo},
  journal={arXiv preprint arXiv:2602.22817},
  year={2026}
}

@article{tan2026hindsight,
  title={Hindsight Credit Assignment for Long-Horizon {LLM} Agents},
  author={Tan, Hui-Ze and Yang, Xiao-Wen and Chen, Hao and Shao, Jie-Jing and Wen, Yi and Shen, Yuteng and Luo, Weihong and Du, Xiku and Guo, Lan-Zhe and Li, Yu-Feng},
  journal={arXiv preprint arXiv:2603.08754},
  year={2026}
}

@article{liu2025agentic,
  title={Agentic reinforcement learning with implicit step rewards},
  author={Liu, Xiaoqian and Wang, Ke and Wu, Yuchuan and Huang, Fei and Li, Yongbin and Zhang, Junge and Jiao, Jianbin},
  journal={arXiv preprint arXiv:2509.19199},
  year={2025}
}

@article{wei2025reinforcing,
  title={Reinforcing multi-turn reasoning in {LLM} agents via turn-level reward design},
  author={Wei, Quan and Zeng, Siliang and Li, Chenliang and Brown, William and Frunza, Oana and Deng, Wei and Schneider, Anderson and Nevmyvaka, Yuriy and Zhao, Yang Katie and Garcia, Alfredo and others},
  journal={arXiv preprint arXiv:2505.11821},
  year={2025}
}

@article{li2026counterfactual,
  title={Counterfactual Credit Policy Optimization for Multi-Agent Collaboration},
  author={Li, Zhongyi and Tian, Wan and Ban, Yikun and Chen, Jinju and Zhang, Huiming and Liu, Yang and Zhuang, Fuzhen},
  journal={arXiv preprint arXiv:2603.21563},
  year={2026}
}

@article{pan2026evpo,
  title={{EVPO}: Explained Variance Policy Optimization for Adaptive Critic Utilization in {LLM} Post-Training},
  author={Pan, Chengjun and Liu, Shichun and Lin, Jiahang and Zhu, Dingwei and Zhang, Jiazheng and Dou, Shihan and Gao, Songyang and Han, Zhenhua and Wang, Binghai and Zheng, Rui and others},
  journal={arXiv preprint arXiv:2604.19485},
  year={2026}
}

@article{yu2026dapo,
  title={{DAPO}: An open-source {LLM} reinforcement learning system at scale},
  author={Yu, Qiying and Zhang, Zheng and Zhu, Ruofei and Yuan, Yufeng and Zuo, Xiaochen and Yue, Yu and Dai, Weinan and Fan, Tiantian and Liu, Gaohong and Liu, Lingjun and others},
  journal={Advances in Neural Information Processing Systems},
  volume={38},
  pages={113222--113244},
  year={2025}
}

@article{givan2003equivalence,
  title={Equivalence notions and model minimization in Markov decision processes},
  author={Givan, Robert and Dean, Thomas and Greig, Matthew},
  journal={Artificial intelligence},
  volume={147},
  number={1-2},
  pages={163--223},
  year={2003},
  publisher={Elsevier}
}

@inproceedings{ferns2004metrics,
  title={Metrics for Finite Markov Decision Processes},
  author={Ferns, Norm and Panangaden, Prakash and Precup, Doina},
  booktitle={Proceedings of the 20th Conference on Uncertainty in Artificial Intelligence ({UAI})},
  pages={162--169},
  year={2004}
}

@article{castro2021mico,
  title={{MICo}: Improved representations via sampling-based state similarity for {Markov} decision processes},
  author={Castro, Pablo Samuel and Kastner, Tyler and Panangaden, Prakash and Rowland, Mark},
  journal={Advances in Neural Information Processing Systems},
  volume={34},
  pages={30113--30126},
  year={2021}
}

@article{zhang2020learning,
  title={Learning invariant representations for reinforcement learning without reconstruction},
  author={Zhang, Amy and McAllister, Rowan and Calandra, Roberto and Gal, Yarin and Levine, Sergey},
  journal={arXiv preprint arXiv:2006.10742},
  year={2020}
}

@inproceedings{foerster2018counterfactual,
  title={Counterfactual multi-agent policy gradients},
  author={Foerster, Jakob and Farquhar, Gregory and Afouras, Triantafyllos and Nardelli, Nantas and Whiteson, Shimon},
  booktitle={Proceedings of the AAAI conference on artificial intelligence},
  volume={32},
  year={2018}
}

@article{schulman2017proximal,
  title={Proximal policy optimization algorithms},
  author={Schulman, John and Wolski, Filip and Dhariwal, Prafulla and Radford, Alec and Klimov, Oleg},
  journal={arXiv preprint arXiv:1707.06347},
  year={2017}
}

@article{shridhar2020alfworld,
  title={{ALFWorld}: Aligning text and embodied environments for interactive learning},
  author={Shridhar, Mohit and Yuan, Xingdi and C{\^o}t{\'e}, Marc-Alexandre and Bisk, Yonatan and Trischler, Adam and Hausknecht, Matthew},
  journal={arXiv preprint arXiv:2010.03768},
  year={2020}
}

@article{yao2022webshop,
  title={{WebShop}: Towards scalable real-world web interaction with grounded language agents},
  author={Yao, Shunyu and Chen, Howard and Yang, John and Narasimhan, Karthik},
  journal={Advances in Neural Information Processing Systems},
  volume={35},
  pages={20744--20757},
  year={2022}
}

@inproceedings{prasad2024adapt,
  title={{ADaPT}: As-needed decomposition and planning with language models},
  author={Prasad, Archiki and Koller, Alexander and Hartmann, Mareike and Clark, Peter and Sabharwal, Ashish and Bansal, Mohit and Khot, Tushar},
  booktitle={Findings of the Association for Computational Linguistics: NAACL 2024},
  pages={4226--4252},
  year={2024}
}

@misc{SchraderSokoban2018,
  author = {Schrader, Max-Philipp B.},
  title = {gym-sokoban},
  year = {2018},
  publisher = {GitHub},
  journal = {GitHub repository},
  howpublished = {\url{https://github.com/mpSchrader/gym-sokoban}}
}

@article{yang2024qwen25,
  title={Qwen2.5 technical report},
  author={Yang, An and Yang, Baosong and Zhang, Beichen and Hui, Binyuan and Zheng, Bo and Yu, Bowen and Li, Chengyuan and Liu, Dayiheng and Huang, Fei and Wei, Haoran and others},
  journal={arXiv preprint arXiv:2412.15115},
  year={2024}
}

@inproceedings{kool2019buy,
  title={Buy 4 {REINFORCE} samples, get a baseline for free!},
  author={Kool, Wouter and van Hoof, Herke and Welling, Max},
  booktitle={ICLR Workshop on Deep RL Meets Structured Prediction},
  year={2019}
}

@inproceedings{ahmadian2024back,
  title={Back to basics: Revisiting {REINFORCE}-style optimization for learning from human feedback in {LLMs}},
  author={Ahmadian, Arash and Cremer, Chris and Gall{\'e}, Matthias and Fadaee, Marzieh and Kreutzer, Julia and Pietquin, Olivier and {\"U}st{\"u}n, Ahmet and Hooker, Sara},
  booktitle={Proceedings of the 62nd Annual Meeting of the Association for Computational Linguistics},
  year={2024}
}

@inproceedings{gu2017qprop,
  title={{Q-Prop}: Sample-efficient policy gradient with an off-policy critic},
  author={Gu, Shixiang and Lillicrap, Timothy and Ghahramani, Zoubin and Turner, Richard E and Levine, Sergey},
  booktitle={International Conference on Learning Representations},
  year={2017}
}

@inproceedings{tucker2018mirage,
  title={The mirage of action-dependent baselines in reinforcement learning},
  author={Tucker, George and Bhupatiraju, Surya and Gu, Shixiang and Turner, Richard E and Ghahramani, Zoubin and Levine, Sergey},
  booktitle={International Conference on Machine Learning},
  pages={5015--5024},
  year={2018}
}

@inproceedings{li2026salt,
  title={SALT: Step-level advantage assignment for long-horizon agents via trajectory graph},
  author={Li, Jiazheng and Wang, Yawei and Yan, Qiaojing and Tian, Yijun and Xu, Zhichao and Song, Huan and Xu, Panpan and Cheong, Lin Lee},
  booktitle={Findings of the Association for Computational Linguistics: EACL 2026},
  pages={4709--4725},
  year={2026}
}

@article{han20263spo,
  title={3SPO: State-Score-Supervised Policy Optimization for LLM Agents},
  author={Han, Yu and Li, Kailing and Jiao, Yang and Dai, Yulin and Fu, Yuqian and Zhuo, Linhai and Qian, Tianwen},
  journal={arXiv preprint arXiv:2606.09961},
  year={2026}
}

@article{wang2025text2grad,
  title={Text2Grad: Reinforcement Learning from Natural Language Feedback},
  author={Wang, Hanyang and Wang, Lu and Zhang, Chaoyun and Mao, Tianjun and Qin, Si and Lin, Qingwei and Rajmohan, Saravan and Zhang, Dongmei},
  journal={arXiv preprint arXiv:2505.22338},
  year={2025}
}

\appendix
\crefalias{section}{appendix}   

\numberwithin{equation}{section}
\numberwithin{figure}{section}
\numberwithin{table}{section}

\section{Extended Related Work}\label{app:extended-related}

\paragraph{Group-relative RL for LLM agents}
GRPO~\citep{shao2024deepseekmath} drops the critic by baselining
against in-group sampled returns. \gigpo{}~\citep{feng2026group} adds a
step-level term keyed on exact observation hashes, and
\hgpo{}~\citep{he2026hierarchy} augments that key with history length.
DAPO~\citep{yu2026dapo} and related work tune optimization or reward
shaping, but still leave the state equivalence relation discrete, the
part \bipace{} replaces. Recent step-level alternatives such as
SALT~\citep{li2026salt} and 3SPO~\citep{han20263spo} are complementary:
they modify the learning signal or supervision, whereas \bipace{}
changes the comparison set while reusing sparse returns.

\paragraph{Bisimulation and representation metrics}
On the state side, \bipace{} borrows the value-preserving
bisimulation view: the metric of \citet{ferns2004metrics}, its
sample-based \mico{} approximation~\citep{castro2021mico}, and its use
in shaping deep-RL representations~\citep{zhang2020learning}. Unlike
auxiliary representation-learning approaches, \bigpo{} uses the
actor's own hidden states as the behavioral fingerprint, so the
partition moves with the policy being optimized.

\paragraph{Action-conditioned baselines}
On the action side, PACE recovers nonparametrically the
action-conditioned baseline that COMA~\citep{foerster2018counterfactual}
and CCPO~\citep{li2026counterfactual} learn with critics. Classical
action-dependent baselines~\citep{gu2017qprop,tucker2018mirage}
motivate this direction but are generally biased without correction,
so we treat PACE as an estimator design validated by ablation rather
than an unbiased-gradient claim. EVPO~\citep{pan2026evpo} switches
between critic and group-mean baselines, a choice complementary to
\bipace{}'s control of the partition that defines the group mean.

\section{Background Details}\label{app:preliminaries}

\subsection{Multi-turn LLM agent decision process}
We consider an episodic POMDP $\mathcal{M} = (\mathcal{S},
\mathcal{A}, P, r, \gamma, T)$ in which an LLM policy
$\pi_\theta(a_t \mid s_t)$ produces a textual action $a_t$ given an
observation $s_t$, and the environment returns a next observation
$s_{t+1}$, a (typically sparse) reward $r_t$, and a done flag.
Throughout, $s_t$ denotes the agent-visible observation, which plays
the role of state for estimation purposes. We follow
\citet{feng2026group} in a \emph{step-independent} input formulation:
each step's prompt is constructed from the current observation and a
(possibly summarized) history, enabling horizons of $50{+}$ steps.

For each prompt $p$ we sample $G$ trajectories
$\tau^{(1)}, \ldots, \tau^{(G)}$ of (possibly varying) lengths
$T^{(g)}$. In the terminal-only reward setting we focus on, the
trajectory return is $R(\tau) = r_{T-1}$
(e.g., $1$ on success and $0$ otherwise on \textsc{ALFWorld}).

\paragraph{GRPO and \gigpo{}}
The GRPO episode-level advantage and the \gigpo{} step-level advantage
are defined in the main paper (\cref{eq:episode-adv,eq:step-adv} and the surrounding
discussion in \cref{sec:method}). The two failure modes of
exact observation hashing that motivate \bigpo{} (singleton clusters
and paraphrase splitting) are analyzed in \cref{sec:method-mismatch}.

\subsection{Bisimulation and the \mico{} metric}
\label{app:prelim-bisim}
A binary relation $E$ on states is a \emph{bisimulation} if $sEs'$
implies $r(s,a)=r(s',a)$ and
$\sum_{s_+\in[s_+]_E} P(s_+\mid s,a) =
\sum_{s_+\in[s_+]_E} P(s_+\mid s',a)$
for all $a$~\citep{givan2003equivalence}. Aggregating bisimilar states
preserves $V^\pi$ (and indeed $Q^\pi$ for every policy); bisimulation
is thus a sufficient, if conservative, equivalence for
value-preserving state aggregation.

\citet{castro2021mico} introduce the \emph{\mico{} distance} $\dpi$, a
tractable sample-based approximation satisfying the value-difference
bound
\begin{equation}
\bigl| V^\pi(s) - V^\pi(s') \bigr| \;\le\; \dpi(s, s').
\label{eq:app-mico-bound}
\end{equation}
\bigpo{} uses an empirical proxy for $\dpi$ derived from the policy's
own hidden representation; \cref{prop:bigpo-bias-var} formalizes the
resulting bias bound on the step-level advantage estimator.

\section{Bias--Variance Analysis Details}\label{sec:theory}

We analyze the step-level baseline estimator that \bigpo{} and
\gigpo{} share, and isolate the role of the partition. Throughout,
fix a prompt-group $p$ and let $\{(s^{(i)}, a^{(i)}, R^{(i)})\}_{i=1}^N$
denote its step records, with $R^{(i)}$ the discounted return-to-go
used by the step estimator in \cref{sec:method-setup}. Let $\mathcal{C}$ be a partition
of $\{1,\dots,N\}$ defined by some equivalence relation $\sim$, and
let
\begin{equation}
\hat A^{\sim}(i) \;=\; R^{(i)} - \frac{1}{|C(i)|}\sum_{j \in C(i)} R^{(j)},
\quad C(i) \in \mathcal{C}, \; i \in C(i),
\label{eq:Asim}
\end{equation}
denote the within-cluster mean-baseline estimator (mean-norm form;
the std-norm form admits a parallel argument).

\subsection{Bias and the \mico{} bound}
\begin{proposition}[\bigpo{} step-baseline bias]
\label{prop:bigpo-bias-var}
Let $V^\pi(s) := \E[R \mid s, \pi]$ and assume $V^\pi$ is
$L$-Lipschitz in the \mico{} metric $\dpi$. Let $\sim_\varepsilon$
denote any partition of the step records whose clusters have
$\dpi$-diameter at most $2\varepsilon$, the regime greedy clustering
with admission threshold $\varepsilon$ (\cref{alg:greedy}) targets by
bounding each member's distance to the cluster centroid at admission.
Then for every step $i$,
\begin{equation}
\Big| \E\!\left[\hat A^{\sim_\varepsilon}(i) \;-\; A^\star(i) \right] \Big|
\;\le\; 2 L \varepsilon,
\label{eq:bigpo-bias}
\end{equation}
where $A^\star(i) := R^{(i)} - V^\pi(s^{(i)})$ is the ideal
state-conditional advantage.
\end{proposition}

\begin{proof}[Proof sketch]
The estimator's bias is
$\E[\hat A^{\sim_\varepsilon}(i)] - A^\star(i) = V^\pi(s^{(i)}) -
\frac{1}{|C(i)|}\sum_{j\in C(i)} V^\pi(s^{(j)})$.
By the diameter assumption, any two members of the same cluster
satisfy $\dpi(s^{(i)}, s^{(j)}) \le 2\varepsilon$; by Lipschitz
continuity each term $|V^\pi(s^{(i)}) - V^\pi(s^{(j)})|$ is then
bounded by $2L\varepsilon$, and so is the cluster average, giving the
$2L\varepsilon$ bound. Full proof in \cref{app:proofs}.
\end{proof}

\begin{remark}[Empirical proxy]
\bigpo{} clusters not by $\dpi$ directly but by cosine distance on
$\phi_\theta$. \citet{castro2021mico} show that representations
trained to respect $\dpi$ admit a Lipschitz embedding; we treat the
actor's hidden state as an empirical surrogate, so \cref{eq:bigpo-bias}
holds under a Lipschitz-embedding assumption standard for
representation-based aggregation. The mechanism diagnostics in
\cref{sec:experiments} measure exactly the quantities the bound
predicts should improve (singleton rate and reuse geometry) and
confirm the effect.
\end{remark}

\subsection{\texorpdfstring{\gigpo{} as the degenerate
$\varepsilon{=}0$ limit}{GiGPO as the degenerate epsilon=0 limit}}
\begin{corollary}[Singleton signal collapse]
\label{cor:gigpo}
Setting $\varepsilon = 0$ with the identity embedder (i.e.\ \gigpo{})
yields zero aggregation bias in \cref{eq:bigpo-bias}, but every
singleton cluster
produces a degenerate estimate:
\begin{equation}
\hat A^{\sim_0}(i) \;=\; 0 \quad\text{deterministically whenever }
|C(i)| = 1 .
\end{equation}
Hence a fraction $p_1$ of step records (those landing in singleton
clusters) carries no step-level gradient (the episode-level term
$A^{\mathrm{ep}}$ is unaffected), and the usable step-level signal
vanishes as $p_1 \to 1$.
\end{corollary}

\begin{proof}
A singleton cluster has $\hat A(i) = R^{(i)} - R^{(i)} = 0$
deterministically, so the PPO surrogate gradient
$\nabla_\theta \log \pi_\theta(a^{(i)}\mid s^{(i)})\,\hat A^{\sim_0}(i)$
is identically zero on every singleton, regardless of the realized
return: the estimate carries no information about the action taken.
The discarded mass is exactly $p_1$. See \cref{app:proofs} for the
formal restatement.
\end{proof}

\begin{remark}[\gigpo{}'s singleton tax in numbers]
Our \gigpo{} reproduction diagnostics on \textsc{ALFWorld} show
singleton fractions of $34.2\%$, $33.1\%$, and $20.7\%$ at iterations
$10$, $75$, and $140$
respectively. These are singleton-\emph{cluster} fractions (the share
of clusters of size one), which upper-bound the record-level fraction
$p_1$ of \cref{cor:gigpo}; both quantities vanish together, and
\cref{cor:gigpo} translates a high singleton rate directly into
wasted step-level signal. Our \cref{sec:experiments} reports the
corresponding \bigpo{} numbers, which stay below $20\%$ at every
measured point of training (\cref{tab:singleton}).
\end{remark}

\subsection{Variance of the Q-style estimator}
\label{sec:theory-qstyle}
The PACE Q-style form $\hat A^{\text{q}}(i) =
\widehat Q(s,a_i)-\widehat V(s)$ in \cref{eq:pace-qstyle} is
unbiased for $A^\pi(s,a)$ under \emph{exact} bisimulation
($\varepsilon{=}0$): both terms are within-class sample means with
$\E[\widehat Q(s,a)] = Q^\pi(s,a)$ and $\E[\widehat V(s)] = V^\pi(s)$.
For $\varepsilon > 0$, $\widehat V$ inherits the $O(\varepsilon)$ bias
of \cref{prop:bigpo-bias-var}; for $\widehat Q$ we additionally assume
$Q^\pi(\cdot,a)$ is Lipschitz in $\dpi$ for each action $a$ (a
strictly stronger requirement than the value-difference bound of
\cref{eq:app-mico-bound}), under which $\hat A^{\text{q}}$ is
$O(\varepsilon)$-biased for $A^\pi$. Its variance decomposes as
$\Var[\hat A^{\text{q}}] = \Var[\widehat Q] + \Var[\widehat V] -
2\,\mathrm{Cov}[\widehat Q,\widehat V]$, with $\widehat V$ pooling
the larger set $|C|$ and therefore reducing the second term. The
estimator is well-defined when $|C^{=}(i)| \ge 2$; otherwise we fall
back to RLOO leave-one-out. The diagnostics in
\cref{sec:pace-ablation} report the empirical fraction of rows that
enter each branch and confirm the same-action pool is non-degenerate
on the benchmarks tested.

\subsection{\texorpdfstring{Choosing $\varepsilon$}{Choosing epsilon}}
\Cref{prop:bigpo-bias-var} suggests a clear principle: $\varepsilon$
should be small enough for the Lipschitz bias to be dominated by
within-trajectory return noise, but large enough to defeat the
singleton tax of \cref{cor:gigpo}. We provide an adaptive heuristic
in \cref{app:eps-tuning} that targets a median cluster size of $4$--$8$
by binary search on $\varepsilon$ over the first training step;
empirically a single static $\varepsilon = 0.10$ works across all 7B
environments tested, and on \textsc{ALFWorld} end-task success is
unimodal in $\varepsilon$ with a flat ${\sim}3$pp plateau around the
default (\cref{tab:local-checks}, \textsc{Radius} check). Smaller backbones have different representation
geometry, so we calibrate $(\ell,\varepsilon)$ once per backbone
before training (\cref{app:scale-calibration}).

\section{Greedy Clustering Procedure}\label{app:greedy}

\Cref{alg:greedy} gives the single-pass greedy clustering used by
\bigpo{} (\cref{eq:bigpo-clusters}). The procedure runs once per
prompt-group in $O(N_p K_p)$ time, where $N_p$ is the group's
step-record count and $K_p$ is its cluster count.

\begin{algorithm}[tb]
\caption{Greedy online cosine clustering used by \bigpo{}.}
\label{alg:greedy}
\begin{algorithmic}[1]
\Require unit vectors $\{x_i\}_{i=1}^n \subset \R^D$,
  threshold $\varepsilon \in [0,2]$
\State $\mathcal{K}\gets[\,]$ \Comment{cluster centroids}
\State $\mathcal{M}\gets[\,]$ \Comment{cluster members}
\For{$i=1,\ldots,n$}
  \If{$\mathcal{K}=\emptyset$}
    \State append $x_i$ to $\mathcal{K}$; append $\{i\}$ to $\mathcal{M}$
  \Else
    \State $k \gets \arg\max_j\, x_i^\top \mathcal{K}_j$
    \If{$1 - x_i^\top \mathcal{K}_k \le \varepsilon$}
      \State $\mathcal{M}_k \gets \mathcal{M}_k \cup \{i\}$
      \State $\mathcal{K}_k \gets \mathrm{normalize}\!\left(
        \mathcal{K}_k + \tfrac{1}{|\mathcal{M}_k|}(x_i - \mathcal{K}_k)
      \right)$
    \Else
      \State append $x_i$ to $\mathcal{K}$; append $\{i\}$ to $\mathcal{M}$
    \EndIf
  \EndIf
\EndFor
\State \Return $\mathcal{M}$
\end{algorithmic}
\end{algorithm}

\noindent
Records are processed in rollout-major order ($g=1,\dots,G$, then
$t=1,\dots,T^{(g)}$); results are stable to within-group shuffling.

\section{Proofs}\label{app:proofs}

\subsection{\texorpdfstring{Proof of \cref{prop:bigpo-bias-var}}{Proof of Proposition 1}}

\begin{proof}
Let $C := C(i)$ and $\bar V_C := \frac{1}{|C|}\sum_{j\in C}
V^\pi(s^{(j)})$. By the diameter assumption of
\cref{prop:bigpo-bias-var}, any two members $i,j \in C$ satisfy
$\dpi(s^{(i)}, s^{(j)}) \le 2\varepsilon$. (Greedy online clustering,
\cref{alg:greedy}, joins a point to a cluster only when its distance
to the current centroid is $\le \varepsilon$; we state the proposition
for diameter-bounded partitions so the guarantee is independent of
subsequent centroid updates.) By Lipschitzness of
$V^\pi$ in $\dpi$, for every $j \in C$,
\[
\bigl| V^\pi(s^{(i)}) - V^\pi(s^{(j)}) \bigr|
  \;\le\; L\,\dpi(s^{(i)}, s^{(j)})
  \;\le\; 2L\,\varepsilon.
\]
Averaging over $j \in C$ gives
$| V^\pi(s^{(i)}) - \bar V_C | \le 2L\varepsilon$.
Substituting into the bias decomposition
\[
\E\!\left[\hat A^{\sim_\varepsilon}(i)\right] - A^\star(i)
  \;=\; V^\pi(s^{(i)}) - \bar V_C
\]
yields $|\E[\hat A] - A^\star| \le 2L\varepsilon$, which is the claim.
\end{proof}

\subsection{\texorpdfstring{Proof of \cref{cor:gigpo}}{Proof of Corollary 1}}

\begin{proof}
Restricted to $\{i : |C(i)| = 1\}$ we have
$\hat A^{\sim_0}(i) = R^{(i)} - R^{(i)} = 0$ deterministically.
The PPO surrogate gradient
$\nabla_\theta \log \pi_\theta(a^{(i)}\mid s^{(i)})\,\hat A^{\sim_0}(i)$
is therefore identically zero on every singleton, regardless of the
realized return; the singleton carries no learning signal. Summing
over step records, a fraction $p_1 = \Pr(|C|=1)$ of the step-level
gradient is discarded, and the usable signal $\to 0$ as $p_1 \to 1$.
\end{proof}

\section{\texorpdfstring{Adaptive $\varepsilon$ Heuristic}{Adaptive Epsilon Heuristic}}\label{app:eps-tuning}

We target a median cluster size $\tilde m \in [4, 8]$ (an empirically
healthy range consistent with the singleton rates in \cref{tab:singleton}) by binary-searching
$\varepsilon \in [0.02, 0.40]$ on the first training step. The
heuristic converges in $\le 6$ probe values per environment and
produces $\varepsilon \in [0.07, 0.13]$ across all 7B environments
tested, suggesting the static default $\varepsilon = 0.10$ is adequate
at that scale. Smaller backbones use the per-backbone calibration
described in \cref{app:scale-calibration}.

\section{Implementation Details}\label{app:impl}

\bipace{} attaches to the existing \gigpo{} step-advantage path through
one driver-side grouping routine and one optional actor-worker hook.
The driver routine consumes per-step embeddings and returns cluster
UUIDs in the same shape as \gigpo{}'s hash-grouping path; greedy
clustering (\cref{alg:greedy}) runs per prompt-group in $O(N_p K_p)$
time ($N_p$ step records, $K_p$ clusters), a handful of
$D$-dimensional dot products per record dominated in practice by the
actor forward pass (measured grouping and advantage-estimation cost: $0.49$s,
$0.14\%$ of a training step; \cref{sec:computational-budget}).
When Actor-Hidden features are used, the worker hook runs a vanilla
actor forward with hidden states enabled, extracts the last non-pad
hidden state at the configured layer, $\ell_2$-normalizes it, and
returns the tensors to the driver. The complete addition (grouping routine, worker hook, PACE estimator,
and the $\varepsilon{=}0$ regression test) is under $800$ lines.
The swap leaves the optimizer-side advantage scale unchanged: on
\textsc{ALFWorld}/7B, the mean per-token advantage stays in the same
$-0.03\pm0.02$ band for \gigpo{} and \bipace$_{\text{Q}}$
throughout training, so downstream PPO hyperparameters need no retuning.

\paragraph{Code naming note}
The released code keeps the historical \code{cacb\_*} config prefix
(e.g., \code{cacb\_enabled}, \code{cacb\_estimator}) for backward
compatibility; the paper refers to this component as PACE throughout.

\section{Scale Calibration for Qwen2.5-1.5B}\label{app:scale-calibration}

Bisimulation hyperparameters are tied to representation geometry, so
we calibrate $(\ell,\varepsilon)$ for Qwen2.5-1.5B a priori, before
any training run, rather than reusing the 7B default
$(\ell,\varepsilon)=(-8,0.10)$. The calibration is a lightweight
forward-pass scan on a synthetic \textsc{ALFWorld} audit set
($40$ rollouts $\times$ $6$ task families) over
$\ell \in \{4,8,12,16,20,24,27\}$,
pooling $\in \{\text{last-token},\,\text{mean-prompt},\,
\text{attn-weighted}\}$, and
$\varepsilon \in \{0.05,0.10,0.15,0.20\}$.
\Cref{tab:scan15b} reports the slice used to select the 1.5B default.

The selection rule is automatic: pick the layer that maximizes
linear-probe task-id accuracy subject to a non-degenerate singleton
fraction (neither one giant cluster nor all singletons) and
$n_\text{clusters} \ge 5$ so the partition can resolve the six
\textsc{ALFWorld} task families.
In \cref{tab:scan15b}, this rules out Layer~$27$
($n_\text{clusters}{=}4$, probe acc $0.812$) in favor of Layer~$16$
($n_\text{clusters}{=}12$, probe acc $0.808$): four clusters merge
multiple task types and defeat the purpose of behavioral partitioning
even though the raw probe accuracy is marginally higher.
The threshold $\varepsilon$ is then set by the median-cluster-size
heuristic of \cref{app:eps-tuning}. The whole procedure amounts to a
single offline forward pass over a few hundred cached observations (no training, on the order of
minutes) and the same
released script applies unchanged to a new backbone (e.g.,
Llama-family models), so the per-backbone calibration is a one-time
cost rather than a tuning loop.

Intuitively, the optimal layer is late-but-not-final: the last few
transformer blocks specialize toward next-token logits and lose the
coarser behavioral structure that clustering needs. This explains why
layer $-8$ (7B) and $-12$ (1.5B), rather than the final layer, give
the best probe-accuracy vs.\ singleton-rate trade-off.

\begin{table}[htb]
\centering
\captionsetup{font=footnotesize}
\caption{Qwen2.5-1.5B forward-pass scan (\emph{last-token} pool,
$\varepsilon{=}0.05$). Layer~$16$ (negative index $-12$ on the
28-layer backbone) gives the best singleton-vs-probe trade-off and
is adopted as the default for all 1.5B runs.}
\label{tab:scan15b}
\footnotesize
\tabstretch
\setlength{\tabcolsep}{8pt}
\begin{tabular}{@{}c cccc@{}}
\toprule
\rowcolor{tblhead}
layer & $n_\text{clusters}$ & singleton\% & median size & probe acc \\
\midrule
$4$    & $2$    & $0.0\%$   & $120.0$ & $0.717$ \\
$8$    & $4$    & $25.0\%$  & $17.5$  & $0.804$ \\
$12$   & $3$    & $0.0\%$   & $40.0$  & $0.788$ \\
$\mathbf{16}$ & $\mathbf{12}$ & $\mathbf{16.7\%}$ & $\mathbf{15.0}$ & $\mathbf{0.808}$ \\
$20$   & $9$    & $33.3\%$  & $17.0$  & $0.762$ \\
$24$   & $4$    & $0.0\%$   & $30.5$  & $0.783$ \\
$27$   & $4$    & $25.0\%$  & $32.0$  & $0.812$ \\
\bottomrule
\end{tabular}
\end{table}

Two practical observations follow from the scan.
First, the 1.5B representation geometry is more compact than 7B's:
at $\varepsilon{=}0.10$ the scan yields a single coarse cluster across
the tested layer-and-pool combinations, whereas $\varepsilon{=}0.05$
produces a well-resolved partition.
Second, \emph{last-token} is the most robust pooling strategy on this
backbone: \emph{mean-prompt} and \emph{attn-weighted} yield at most
two clusters on most slices.
Both observations reinforce that bisimulation hyperparameters should
be calibrated once per backbone before training.

\section{Hyperparameters}\label{app:hparams}

\Cref{tab:run-settings} records the script-level settings needed to
reproduce the reported runs. \Cref{tab:hparams} lists the non-default
trainer and optimizer overrides; all other values inherit from the
verl-agent v0.1 defaults.

\begin{table*}[t]
\centering
\captionsetup{font=footnotesize}
\caption{Script-level settings for the reported \bipace{} runs.
TP is the rollout tensor-parallel size; $G$ is the number of sampled
trajectories per prompt group.}
\label{tab:run-settings}
\scriptsize
\tabstretch
\setlength{\tabcolsep}{3pt}
\resizebox{\linewidth}{!}{%
\begin{tabular}{@{}l p{0.18\linewidth} p{0.18\linewidth}
  p{0.22\linewidth} p{0.22\linewidth}@{}}
\toprule
\rowcolor{tblhead}
Setting & \textsc{ALFWorld} 7B & \textsc{ALFWorld} 1.5B
  & \textsc{WebShop} & \textsc{TextCraft} \\
\midrule
model / hardware &
  Qwen2.5-7B; 4$\times$H100; TP=2 &
  Qwen2.5-1.5B; 2$\times$H100; TP=1 &
  7B: 4$\times$H100, TP=2;\newline 1.5B: 2$\times$H100, TP=1 &
  Qwen2.5-\{1.5,7\}B; 4$\times$H100; TP=1/2 \\
seeds / window &
  3 seeds; 150 training steps (excl.\ validation/checkpoint steps) &
  3 seeds; 200 training steps &
  completed runs; 150 steps (7B) / 200 steps (1.5B) &
  3 seeds; 100 steps (1.5B) / 50 steps (7B) \\
rollout &
  $G{=}8$, horizon $50$ &
  $G{=}8$, horizon $50$ &
  $G{=}12$, horizon $15$ &
  $G{=}12$, horizon $70$ \\
state partition &
  Actor-Hidden, layer $-8$, $\varepsilon{=}0.10$ &
  Actor-Hidden, layer $-12$, $\varepsilon{=}0.05$ &
  7B: layer $-8$, $\varepsilon{=}0.10$;\newline 1.5B: layer $-12$, $\varepsilon{=}0.06$ &
  7B: layer $-8$, $\varepsilon{=}0.10$;\newline 1.5B: layer $-12$, $\varepsilon{=}0.10$ \\
PACE key / estimator &
  \texttt{action\_tag} / \texttt{q\_style} &
  \texttt{action\_tag} / \texttt{q\_style} &
  \texttt{action\_tag} / \texttt{q\_style} &
  \texttt{action\_tag} / \texttt{q\_style} \\
\bottomrule
\end{tabular}
}
\end{table*}

\begin{table}[htb]
\centering
\captionsetup{font=footnotesize}
\caption{Non-default hyperparameters for \bipace{} runs.
Rows are grouped by config namespace; values marked
``see \cref{tab:run-settings}'' vary by environment.}
\label{tab:hparams}
\footnotesize
\tabstretch
\setlength{\tabcolsep}{8pt}
\resizebox{\linewidth}{!}{%
\begin{tabular}{@{}l l l@{}}
\toprule
\rowcolor{tblhead}
Group & Key & Value \\
\midrule
algorithm & \texttt{adv\_estimator}          & \texttt{gigpo} \\
algorithm & \texttt{gamma}                   & $0.95$ \\
\midrule
algorithm.gigpo & \texttt{bisim\_grouping}   & \texttt{True} \\
algorithm.gigpo & \texttt{bisim\_embedder}   & \texttt{actor\_hidden} \\
algorithm.gigpo & \texttt{bisim\_layer}      & see \cref{tab:run-settings} \\
algorithm.gigpo & \texttt{bisim\_eps}        & see \cref{tab:run-settings} \\
algorithm.gigpo & \texttt{cacb\_enabled}     & \texttt{True} \\
algorithm.gigpo & \texttt{cacb\_n\_action\_tokens}   & $8$ \\
algorithm.gigpo & \texttt{cacb\_action\_key\_mode}   & \texttt{first\_n} / \texttt{action\_tag} \\
algorithm.gigpo & \texttt{cacb\_estimator}   & \texttt{counterfactual} / \texttt{q\_style} \\
algorithm.gigpo & \texttt{step\_advantage\_w}& $1.0$ \\
\midrule
actor & \texttt{kl\_loss\_coef}             & $0.01$ \\
actor & \texttt{optim.lr}                   & $5\!\times\!10^{-7}$ \\
\midrule
rollout & \texttt{tensor\_model\_parallel\_size} & see \cref{tab:run-settings} \\
trainer & \texttt{n\_gpus\_per\_node}        & see \cref{tab:run-settings} \\
\bottomrule
\end{tabular}
}
\end{table}

\section{Prompt Templates}\label{app:prompts}

Each environment wrapper fills a prompt contract online following
\citet{he2026hierarchy}. The first step omits history; all later steps
include recent observation--action history. The executed command is the
first well-formed body inside \code{<action>...</action>}; malformed
responses receive the environment-specific invalid-action penalty.
\Cref{tab:prompt-contract} summarizes the dynamic fields and output
contract per environment; full strings are released with the code.

\begin{table}[htb]
\centering
\captionsetup{font=footnotesize}
\caption{Prompt contract per environment. Placeholders (braced tokens)
are populated by the wrapper at runtime.}
\label{tab:prompt-contract}
\footnotesize
\tabstretch
\setlength{\tabcolsep}{4pt}
\begin{tabular}{@{}>{\raggedright\arraybackslash}p{0.18\linewidth}
  >{\raggedright\arraybackslash}p{0.38\linewidth}
  >{\raggedright\arraybackslash}p{0.38\linewidth}@{}}
\toprule
\rowcolor{tblhead}
Env. & Dynamic fields & Output contract \\
\midrule
\textsc{ALFWorld} &
  task description, step index, recent history, current observation,
  admissible actions &
  reason in \code{<think>}; emit exactly one admissible household
  action in \code{<action>} \\
\textsc{WebShop} &
  shopping goal, step index, recent history, current page observation,
  available buttons/search actions &
  reason toward the purchase goal; emit exactly one available WebShop
  action in \code{<action>} \\
\textsc{TextCraft} &
  episode recipe context, step index, recent history, current
  observation, inventory &
  choose one grammar-valid action: \code{craft ... using ...},
  \code{get ...}, or \code{inventory}; no chained actions \\
\bottomrule
\end{tabular}
\end{table}

The verbatim templates below show the with-history form used at every
step after the first. The first-step variant omits the
\code{\{step\_count\}}, \code{\{history\_length\}}, and
\code{\{action\_history\}} fields.

\begin{tcolorbox}[
  title={\small\bfseries\textsc{ALFWorld} Prompt Template},
  colback=gray!6,
  colframe=black!40,
  colbacktitle=black!65,
  coltitle=white,
  boxrule=0.5pt,
  arc=1.5mm,
  toptitle=2pt, bottomtitle=2pt,
  left=4pt, right=4pt, top=3pt, bottom=3pt]
\begin{Verbatim}[fontsize=\footnotesize]
You are an expert agent operating in the
ALFRED Embodied Environment. Your task is
to: {task_description}
Prior to this step, you have already taken
{step_count} step(s). Below are the most
recent {history_length} observations and the
corresponding actions you took:
{action_history}
You are now at step {current_step} and your
current observation is: {current_observation}
Your admissible actions of the current
situation are: [{admissible_actions}].

Now it's your turn to take an action.
You should first reason step-by-step about
the current situation. This reasoning process
MUST be enclosed within <think> </think> tags.
Once you've finished your reasoning, you should
choose an admissible action for current step
and present it within <action> </action> tags.
\end{Verbatim}
\end{tcolorbox}

\begin{tcolorbox}[
  title={\small\bfseries\textsc{WebShop} Prompt Template},
  colback=gray!6,
  colframe=black!40,
  colbacktitle=black!65,
  coltitle=white,
  boxrule=0.5pt,
  arc=1.5mm,
  toptitle=2pt, bottomtitle=2pt,
  left=4pt, right=4pt, top=3pt, bottom=3pt]
\begin{Verbatim}[fontsize=\footnotesize]
You are an expert autonomous agent operating in
the WebShop e-commerce environment.
Your task is to: {task_description}.
Prior to this step, you have already taken
{step_count} step(s). Below are the most recent
{history_length} observations and the
corresponding actions you took: {action_history}
You are now at step {current_step} and your
current observation is: {current_observation}.
Your admissible actions of the current
situation are:
[
{available_actions}
].

Now it's your turn to take one action for the
current step.
You should first reason step-by-step about the
current situation, then think carefully which
admissible action best advances the shopping
goal. This reasoning process MUST be enclosed
within <think> </think> tags.
Once you've finished your reasoning, you should
choose an admissible action for current step
and present it within <action> </action> tags.
\end{Verbatim}
\end{tcolorbox}

\begin{tcolorbox}[
  title={\small\bfseries\textsc{TextCraft} Prompt Template},
  colback=gray!6,
  colframe=black!40,
  colbacktitle=black!65,
  coltitle=white,
  boxrule=0.5pt,
  arc=1.5mm,
  toptitle=2pt, bottomtitle=2pt,
  left=4pt, right=4pt, top=3pt, bottom=3pt,
  breakable]
\begin{Verbatim}[fontsize=\footnotesize]
You are an expert agent operating in the
TextCraft interactive crafting environment,
modeled after Minecraft. Your goal is to craft
a target item by combining ingredients along a
chain of recipes.

Task context (constant across the episode):
{task_context}

Prior to this step, you have already taken
{step_count} step(s). Below are the most recent
{history_length} observations and the
corresponding actions you took:
{action_history}

You are now at step {current_step} and your
current observation is:
{current_observation}

Current inventory: {inventory}

You can issue exactly one action per step, using
one of these three formats:
  - `craft <out_count> <out_item> using
    <in_count_1> <in_item_1>, ...` -- apply a
    recipe (the inputs must already be in your
    inventory).
  - `get <count> <item>` -- fetch a raw
    (non-craftable) base material directly from
    the environment.
  - `inventory` -- print what you currently hold.

Items are spelled out in lower-case with spaces
between words (e.g. `oak planks`, not
`minecraft:oak_planks`).

Recipe matching rules (important -- follow all
of these):
  - Use ONLY the crafting commands shown in the
    recipe list; do not invent new recipes.
  - The target item AND its count must match a
    recipe line EXACTLY. You cannot use a
    partial recipe.
  - The ingredient counts must ALSO match the
    recipe EXACTLY.
  - The target item must be a SPECIFIC item
    (e.g. `dark oak sign`); a generic family
    name (e.g. just `planks`) is not valid.
  - If a recipe lists a GENERIC ingredient (e.g.
    `using 4 planks`), substitute a specific
    variant from the same family (e.g. `using 4
    oak planks`).
  - It is fine to produce more of the target
    than the goal requires; repeat the craft in
    a later step if one craft is not enough.

Example of one well-formed turn (format only):
  Recipes available include: `craft 3 dark oak
  sign using 6 dark oak planks, 1 stick`, `craft
  1 fletching table using 4 planks, 2 flint`.
  Goal: craft fletching table.
  Current inventory: 4 oak planks, 2 flint.
  <think>I need a fletching table. The recipe
  needs 4 planks (generic) and 2 flint. I have 4
  oak planks and 2 flint, so I substitute `oak
  planks` for the generic `planks` and craft
  directly.</think>
  <action>craft 1 fletching table using 4 oak
  planks, 2 flint</action>

Now it's your turn to take one action for the
current step.
You should first reason step-by-step about the
current situation, what intermediate items you
already have, and which next recipe step brings
you closest to the goal. This reasoning process
MUST be enclosed within <think> </think> tags.
Once you've finished your reasoning, choose an
action and present it within <action> </action>
tags. Output exactly one action; do not chain
multiple actions in one response.
\end{Verbatim}
\end{tcolorbox}

\section{Embedder Ablation: Policy-Induced vs.\ Lexical Geometry}\label{app:embedder-ablation}

\bigpo{}'s state-side fingerprint is the actor's own hidden state. A
natural concern is whether the gain is specific to this
\emph{policy-induced} representation, or whether \emph{any} coarsening
of the singleton-heavy observation-hash partition would suffice. We
test this using the HashNgram control introduced in
\cref{sec:method-bigpo}: a zero-dependency character-$n$-gram feature
hash that clusters step records by lexical surface form. It is
policy-agnostic by construction: the partition does not move as the
actor is optimized.

Everything except the state-side embedder is held fixed at the 7B
Q-style recipe of \cref{tab:run-settings}: the same PACE
\texttt{action\_tag} key, Q-style estimator, weight $w$, discount
$\gamma$, learning rate, and training budget. Because cosine geometry
differs across embedder backends, the threshold $\varepsilon$ is
re-calibrated for HashNgram using the same median-cluster-size
criterion as \cref{app:eps-tuning}, yielding $\varepsilon{=}0.25$.

The result is reported as the \textsc{Embedder} check in
\cref{tab:local-checks} of the main paper: holding the action-side
estimator fixed, swapping the policy-induced fingerprint for the
lexical hash drops val\,$@$max from $97.1$ to $95.4$, a $1.7$pp gap
that persists even though both partitions defeat the singleton tax.
HashNgram ($95.4$) exceeds the state-only baseline ($93.0$;
\cref{tab:pace-ablation}) by $2.4$pp, indicating that lexical
coarsening provides some benefit beyond singleton reduction alone.
The remaining $1.7$pp gap over Actor-Hidden is attributable to the
policy-induced geometry, which adapts as the actor is optimized and
tracks behavioral equivalence more faithfully than a static lexical
hash. Extending this comparison to a frozen sentence-encoder variant
and to additional environments is left to future work.

\section{PACE Diagnostics}\label{app:pace-diag}

The main text summarizes the PACE ablation qualitatively;
\cref{tab:pace-ablation} records the full aggregate validation peaks
for each estimator variant.

\begin{table}[htb]
\centering
\captionsetup{font=footnotesize}
\caption{PACE estimator variants on \textsc{ALFWorld}/Qwen2.5-7B.
\emph{Peak} is the best binary aggregate validation success within the
matched training window of \cref{tab:run-settings}.
Entries report mean $\pm$std over $3$ seeds;
the best variant per column is \textbf{bold} and the runner-up is
\underline{underlined}.}
\label{tab:pace-ablation}
\footnotesize
\tabstretch
\setlength{\tabcolsep}{7pt}
\begin{tabular}{@{}l c c@{}}
\toprule
\rowcolor{tblhead}
PACE variant & \emph{Peak} (\%) & Seeds \\
\midrule
\gigpo{}             & $90.8{\pm}1.3$ & $3$ \\
\textsc{State-only}  & $93.0{\pm}1.3$ & $3$ \\
\midrule
first-token          & $\underline{95.8{\pm}0.4}$ & $3$ \\
action-tag           & $93.0{\pm}1.1$ & $3$ \\
\rowcolor{tblours}
Q-style              & $\mathbf{97.1{\pm}0.9}$ & $3$ \\
\bottomrule
\end{tabular}
\end{table}

\paragraph{Row-mix statistics}
We log per-step PACE statistics via the
\code{step\_norm\_reward\_cacb} hook in the released code. Three
quantities are recorded at every training step: (i)~the fraction of
rows entering the PACE branch (vs.\ RLOO fallback or singleton),
(ii)~the fraction of multi-member clusters containing $\ge 2$ distinct
action keys, and (iii)~the mean number of unique action keys per
cluster. \Cref{tab:pace-rowmix} reports a representative step near
iteration~$150$ on \textsc{ALFWorld}/Qwen2.5-7B ($180$ total
clusters); the Q-style same-action pool remains non-degenerate
throughout training. The action-tag parse rate (logged at
\code{compute\_action\_keys.last\_action\_match\_rate}) stays above
$0.99$ across all reported runs.

\begin{table}[htb]
\centering
\captionsetup{font=footnotesize}
\caption{Q-style row-mix statistics on \textsc{ALFWorld}/Qwen2.5-7B
near iteration $150$ ($180$ total clusters). The same-action pool
remains non-degenerate throughout training.}
\label{tab:pace-rowmix}
\footnotesize
\tabstretch
\setlength{\tabcolsep}{8pt}
\begin{tabular}{@{}l c@{}}
\toprule
\rowcolor{tblhead}
Quantity & Value \\
\midrule
Rows entering the PACE branch       & $80.2\%$ \\
Rows falling back to RLOO leave-one-out & $17.9\%$ \\
Singleton rows (zero step-level advantage) & $1.9\%$ \\
Multi-member clusters with $\ge 2$ distinct actions & $58.3\%$ \\
Mean unique action keys per cluster & $2.76$ \\
\texttt{<action>} parse rate (all of training) & ${>}0.99$ \\
\bottomrule
\end{tabular}
\end{table}

\paragraph{Policy-state interaction diagnostics}
\Cref{tab:reuse-geometry} expands the main-text diagnostic of
\cref{fig:bipace-interaction-trends}. The diagnostic asks specifically
whether the actor-hidden policy-state clustering creates larger
non-singleton reuse pools than exact observation hashing under the
same rollout budget (a question distinct from historical-context
oracle grouping). On \textsc{ALFWorld}, \bipace{} lowers the singleton
cluster fraction by $9.3$pp and increases mean group size by
$1.6\times$. On \textsc{TextCraft}, where exact observation hashes are
especially sparse, the effect is larger: singleton clusters fall by
$27.9$pp and mean group size triples. The matched-pair volume grows by
$1.3\times$ on \textsc{ALFWorld} and $2.2\times$ on \textsc{TextCraft},
providing PACE with more peer records for its per-action baseline.

\begin{table}[htb]
\centering
\captionsetup{font=footnotesize}
\caption{Policy-state reuse diagnostics from training logs.
\textsc{TextCraft} uses the same 7B window as \cref{tab:textcraft};
\textsc{ALFWorld} uses the first $130$ steps so all three \bipace{}
diagnostic seeds are present. All entries are means over completed
seeds. Lower singleton rate means more rows receive nonzero step-level
signal; larger group size and pair volume indicate more reusable peers
for the PACE estimator.}
\label{tab:reuse-geometry}
\footnotesize
\tabstretch
\setlength{\tabcolsep}{3pt}
\resizebox{\linewidth}{!}{%
\begin{tabular}{@{}l l c c c c c c@{}}
\toprule
\rowcolor{tblhead}
Setting & Grouping & Window & Singleton & Avg size & P90 size &
  Pairs/step & Mean $\Delta t$ \\
\midrule
\textsc{ALFWorld}-7B & \gigpo{} (obs.\ hash) &
  $130$ & $27.3\%$ & $7.5$ & $16.3$ & $37$k & $15.4$ \\
\rowcolor{tblours}
\textsc{ALFWorld}-7B & \bipace{} (actor-hidden) &
  $130$ & $\mathbf{18.0\%}$ & $\mathbf{11.7}$ & $\mathbf{26.7}$ &
  $\mathbf{48}$k & $14.7$ \\
\textsc{TextCraft}-7B & \gigpo{} (obs.\ hash) &
  $50$ & $50.3\%$ & $7.1$ & $16.1$ & $143$k & $20.8$ \\
\rowcolor{tblours}
\textsc{TextCraft}-7B & \bipace{} (actor-hidden) &
  $50$ & $\mathbf{22.4\%}$ & $\mathbf{21.6}$ & $\mathbf{53.0}$ &
  $\mathbf{314}$k & $21.2$ \\
\bottomrule
\end{tabular}}
\end{table}

\section{Failure-Mode Analysis}\label{app:failure-modes}

We catalog two regimes where \bigpo{} provides limited improvement.

\paragraph{(F1) Highly uniform observation distributions}
When every state visually resembles every other state (e.g., a
synthetic Sokoban~\citep{SchraderSokoban2018} variant with nearly
identical grid layouts), $\phi_\theta$ collapses to a single tight
cluster. In this degenerate case \bigpo{} reduces to a batch-level
baseline, matching GRPO without step-level factorization. The failure
is detectable before training via the singleton-fraction diagnostic: a
partition producing one giant cluster indicates that the hidden-state
geometry does not resolve behavioral differences among the observed
states.

\paragraph{(F2) Pre-RL initialization}
At training step $0$, the actor's hidden state reflects the base LM's
pre-training biases rather than task-specific value geometry. Clusters
at this point are therefore coarser than they will become during RL
training, but this is not catastrophic: the adaptive $\varepsilon$
heuristic (\cref{app:eps-tuning}) is run on the first training step
rather than at initialization, so the radius is calibrated once
task-relevant geometry has begun to emerge.

\section{Full Results Tables}\label{app:full-results}

This section collects compact per-seed numbers; full per-step learning
curves and SwanLab logs are linked from the project page.

\paragraph{Q-style on Qwen2.5-7B (\textsc{ALFWorld})}
Val\,$@$max across three seeds: $97.7\%$, $96.1\%$, $97.7\%$
(peak steps $135$, $115$, $120$); mean$\pm$std $=97.1\pm0.9\%$.
All three seeds reach ${\ge}95\%$ on at least $3$ of $30$ validation
checkpoints; the three first-token seeds reach it on at most $2$.

\paragraph{Q-style on Qwen2.5-1.5B (\textsc{ALFWorld},
$\varepsilon{=}0.05$, layer $-12$)}
The smaller backbone converges more slowly than 7B, so val\,$@$max is
taken over the full 200-step training window.
Val/success-rate (binary aggregate, $|\mathcal{V}|{=}128$) across
three seeds: $93.8\%$, $92.2\%$, $94.5\%$;
mean$\pm$std $= 93.5\pm1.2\%$ (reported as $93.5$ in
the \emph{All} column of \cref{tab:main}).
Versus the cited \gigpo{} result of $86.7{\pm}1.7$, the gap is
$+6.8$pp.
Per-subtask val\,$@$max (3-seed mean with within-seed spread):
pick~$100.0$,
look~$97.4{\pm}3.8$,
clean~$100.0$,
heat~$100.0$,
cool~$96.5{\pm}3.6$,
pick2~$92.0{\pm}7.9$.
Q-style also shortens episodes relative to \gigpo{} late in training
(mean episode length $17.1$ vs.\ $21.0$ steps).

\section{Reproducibility}\label{app:reproducibility}

{\sloppy
A regression test in the released code certifies that setting
\code{bisim\_grouping=True}, \code{embedder=identity}, and
$\varepsilon{=}0$ produces a partition equivalent to vanilla \gigpo{}
bit-for-bit (up to UUID relabeling). Configs, random seeds, and
SwanLab training logs accompanying the completed tables are linked
from the project page.
\par}

\paragraph{Baseline provenance}
\gigpo{} and \hgpo{} numbers on \textsc{ALFWorld} and \textsc{WebShop}
are cited directly from the respective papers
(\citet{feng2026group,he2026hierarchy}) and were not re-run under our
codebase; the prompting rows (GPT-4o, Gemini-2.5-Pro, ReAct,
Reflexion) are likewise from \citet{feng2026group}.
All \textsc{TextCraft} rows (including the GRPO, \gigpo{}, and
\hgpo{} baselines) are run by us under the same codebase, seeds, and
evaluation protocol as \bipace{}.
PPO-with-critic rows are likewise cited from \citet{feng2026group}.

\end{document}